\documentclass{article} % For LaTeX2e
\usepackage{nips15submit_e,times}
\usepackage{hyperref}
\usepackage{url}
\usepackage{amsmath}
\usepackage{amsthm}
\theoremstyle{definition}
\usepackage{amsfonts}
\usepackage[small]{caption}
\usepackage{graphicx}
\usepackage{subcaption}
\usepackage{float}
\usepackage{xfrac}
\usepackage{array}
\usepackage{booktabs}
\usepackage{algorithm}
\usepackage{bookmark}
\usepackage{algorithmic}
\usepackage{multicol}
\usepackage{wrapfig}
\usepackage{threeparttable}
\newcommand\numberthis{\addtocounter{equation}{1}\tag{\theequation}}

\newtheorem*{theorem*}{Lemma}
\title{ Multi-Label Learning with Provable Guarantee}

\author{Sayantan Dasgupta \\ \url{sayantad@uci.edu}}

\nipsfinalcopy % Uncomment for camera-ready version

\begin{document}

\maketitle

\begin{abstract}
Here we study the problem of learning labels for large text corpora where each text can be assigned a variable number of labels. The problem might seem trivial when the label dimensionality is small, and can be easily solved using a series of one-vs-all classifiers. However, as the label dimensionality increases to several thousand, the parameter space becomes extremely large, and it is no longer possible to use the one-vs-all technique. Here we propose a model based on the factorization of higher order moments of the words in the corpora, as well as the cross moment between the labels and the words for multi-label prediction. Our model provides guaranteed convergence bounds on the estimated parameters. Further, our model takes only three passes through the training dataset to extract the parameters, resulting in a highly scalable algorithm that can train on GB's of data consisting of millions of documents with hundreds of thousands of labels using a nominal resource of a single processor with 16GB RAM. Our model achieves 10x-15x order of speed-up on large-scale datasets while producing competitive performance in comparison with existing benchmark algorithms.
\end{abstract}

\section{Introduction}
Multi-label learning for large text corpora is an upcoming problem in Large-Scale Machine Learning. Unlike the multi-class classification where a text is assigned only one label from a set of labels, here a text can have a variable number of labels. A basic approach to the problem is to use 1-vs-all classification technique by training a single binary classifier for every label. If the vocabulary size of a text corpus is $D$ and the label dimensionality is $L$, then these 1-vs-all models require $\mathcal{O}(DL)$ parameters. Most of the text corpora have moderate to high vocabulary size ($D$), and 1-vs-all models for label prediction is feasible as long as $L \ll D$. However, as the number of labels increases to a point when $L \sim D$, it is no longer possible to use the 1-vs-all classifier, since the number of parameters required increases to $\mathcal{O}(D^2)$, and the model can no longer be stored in the memory \cite{LEML}. 

Recently there has been attempts to reduce the complexity of such models, by using a low rank mapping $\Phi: \mathbb{R}^D \rightarrow \mathbb{R}^L$ in between the data and the labels. If the rank of such mappings is limited to $K \ll D$, then the model requires $\Theta \left( (L+D)K \right)$ parameters. Both WSABIE \cite{WSABIE} and LEML \cite{LEML} utilizes such mappings. WSABIE defines weighted approximate pair-wise rank (WARP) loss on such mappings and optimizes the loss on the training dataset. LEML uses similar mapping but generalizes the loss function to squared-loss, sigmoid loss or hinge loss, which are typical to the cases of Linear Regression, Logistic Regression, and Linear SVM respectively. 

Both of WSABIE and LEML uses low-rank discriminative models, where the low-rank mapping usually has the form $Z=HW^\top$, where $W \in \mathbb{R}^{D\times K}$ and $H \in \mathbb{R}^{L\times K} $. Here we propose a generative solution for the same problem using latent variable based probabilistic modeling. Unlike the usual cases where such latent variable models are trained using EM, we use Method of Moments \cite{MoM} to extract the parameters from the latent variable model. We show that our method can be globally convergent when the sample size is larger than a specific lower bound, and establish theoretical bounds for the estimated parameters. We also show the competitive performance of our method in terms of classification measures as well as computation time. 

\section{Latent Variable Model}

We use a generative model as shown in Figure \ref{fig:model}. The underlying generative process of the model is described as follows.

\begin{figure}[tb]
\label{fig:model}
\centering
\includegraphics[scale=0.4]{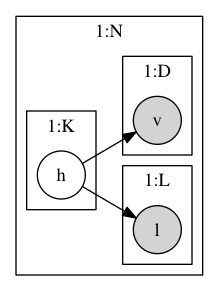} 
\caption{Plate Notation}
\label{fig:model}
\end{figure}

\subsection{Generative Model}

Let us assume that there are $N$ documents, the vocabulary size is $D$, and total number of labels are $L$. For any document $d \in \{ d_1,d_2 \dots d_N \} $ we first choose a latent state of $h \in \{1,2 \dots K\}$ from the discrete distribution $P\big[h \big]$, then we choose a word $v \in \{ v_1,v_2 \dots v_D \}$ from the discrete distribution $P\big[ v|h \big]$, and a label $l \in \{ l_1,l_2 \dots l_L \}$ from the discrete distribution $P\big[ l|h \big]$.  The generative process is as follows:

\begin{gather*} 
h \sim Discrete(P\big[ h \big]) \\
v \sim Discrete(P\big[ v|h \big]) \\
l \sim Discrete(P\big[ l|h \big]) \numberthis \label{1}
\end{gather*}

Let us denote the probability of the latent variable $h$ assuming the state $k \in {1 \dots K} $ as,

\begin{equation}
\pi_k = P\big[ h=k \big]
\end{equation}

Let us define $\mu_k \in \mathbb{R}^D  $ as the probability vector of all the words conditional to the latent state $k \in {1 \dots K} $, i.e. 
\begin{equation}
\mu_k=P\big[ v|h=k \big]
\end{equation}

and $\gamma_k \in \mathbb{R}^L  $ as the probability vector of all the labels conditional to the latent state $k \in {1 \dots K} $, i.e. 
\begin{equation}
\gamma_k=P\big[ l|h=k \big]
\end{equation}

Let the matrix $O \in \mathbb{R}^{D \times K}$ denote the conditional probabilities for the words, i.e.,\\  $O_{i,k}=P\big[ v_i|h=k \big]$.
Then $O=[\mu_1|\mu_2| \dots |\mu_K] $. 

Similarly, let $Q  \in \mathbb{R}^{L \times K}$ denote the conditional probabilities for the labels, i.e., \\ $Q_{j,k}=P\big[ l_j|h=k \big]$. Then, $Q = [\gamma_1|\gamma_2| \dots |\gamma_K] $.

We assume that the matrix $O$ and $Q$ are of full rank, and their columns are fully identifiable. The aim of our algorithm is to estimate the matrices $O$, $Q$ and the vector $\pi$.

\subsection{Moment Formulation}

Following the generative model in equation \ref{1},  we try to formulate the matrix of the joint probability mass function of the words. Let us assume that we choose two words $w_1$ and $w_2$ from a document at random. The probability $P[w_1=v_i]$ represents the probability by which any word picked at random from a document turns out to be $v_i$, and it is nothing but $P[v_i]$.

Similarly, $P[w_1=v_i,w_2=v_j]$ represents the joint probability by which any two words picked at random from a document turn out to be $v_i$ and $v_j$, and it is same as $P[v_i,v_j]$, with $i, j = 1,2 \dots D$.
Now, from the generative process in Equation \ref{1}, $w_1$ and $w_2$ are conditionally independent given $h$, i.e., $P[w_1,w_2|h=k] = P[w_1|h=k]P[w_2|h=k]$ with $k=1,2 \dots K$.

Therefore, the joint probability of two words $v_i$ and $v_j$ is,
\begin{align*}
&P[v_i,v_j]\\
&=P[w_1=v_i,w_2= v_j] \\
&=\sum_{k=1}^K{P[w_1=v_i,w_2= v_j|h=k]P[h=k]} \\
&=\sum_{k=1}^K{P[w_1=v_i|h=k]P[w_2=v_j|h=k]P[h=k]} \\
&=\sum_{k=1}^K{P[v_i|h=k]P[v_j|h=k]P[h=k]} \\
&=\sum_{k=1}^K{\pi_k \mu_{ki} \mu_{kj} }, \qquad \forall i,j \in \{ 1,2 \dots D\}
\end{align*}

Therefore, if we define $M_2 \in \mathbb{R}^{D \times D}$ as the matrix representing the joint probability mass function of the words, with $M_{2_{i,j}}=P\big[v_i,v_j\big]$, we can express it as,
\begin{equation}
M_2=\sum_{k=1}^K{\pi_k\mu_k {\mu_k}^\top } =\sum_{k=1}^K{\pi_k\mu_k \otimes \mu_k} 
\end{equation}

Similarly, if the tensor $M_3 \in \mathbb{R}^{D \times D \times D}$ is defined as the third order probability moment, with $M_{3_{i,j,\tau}} =P[v_i,v_j,v_{\tau}]$  $ \forall i,j,\tau \in \{ 1,2 \dots D\}$, then it can be represented as,

\begin{equation}
M_3=\sum_{k=1}^K{\pi_k\mu_k \otimes \mu_k \otimes \mu_k}
\end{equation}

Further, if we define the cross moment between the labels and the words as $M_{2L} \in \mathbb{R}^{L \times D \times D}$, with  $M_{{2L}_{\tau,i,j}} =P[l_{\tau},v_i,v_j]$, where  $ \tau \in \{1,2 \dots L \}$ and $i,j \in \{ 1,2 \dots D\}$, then

\begin{equation}
M_{2L}=\sum_{k=1}^K{\pi_k  \gamma_k \otimes \mu_k \otimes \mu_k}
\end{equation}

\subsection{Parameter Extraction}

In this section, we revisit the method to extract the matrices $O$ and $Q$ as well as the latent state probabilities $\pi$. The first step is to whiten the matrix $M_2$, where we try to find a matrix low rank $W$ such that $W^\top M_2 W=I$. This is a method similar to the whitening in ICA, with the covariance matrix being replaced with the co-occurrence probability matrix in our case. 

The whitening is usually done through eigenvalue decomposition of $M_2$. If the $K$ maximum eigenvalues of $M_2$ are $\{ \nu_k \}_{k=1}^K$, and the corresponding eigenvectors are $\{ \omega_k \}_{k=1}^K$, then the whitening matrix of rank $K$ is computed as $W=\Omega{\Sigma}^{-1/2}$, where $\Omega=\big[ \omega_1|\omega_2| \dots |\omega_K \big]$, and $\Sigma = diag(\nu_1,\nu_2,\dots,\nu_K)$. 

Upon whitening $M_2$ takes the form 
\begin{align*}
W^\top M_2W &= W^\top \big( \sum_{k=1}^K{\pi_k\mu_k  \mu_k^\top} \big)W = \sum_{k=1}^K{ \big( \sqrt{\pi_k}W^\top \mu_k \big)  \big( \sqrt{\pi_k}W^\top \mu_k \big)^\top  } = \sum_{k=1}^K \tilde{\mu}_k \tilde{\mu}_k^\top  = I 
\numberthis \label{eqn:whiten}
\end{align*}

Hence $ \tilde{\mu}_k = \sqrt{\pi_k}W^\top\mu_k \in \mathbb{R}^K$ are orthonormal vectors in the domain $\mathbb{R}^K$.
Multiplying $M_3$ along all three dimensions by $W$, we get
\begin{align*}
 \tilde{M_3} &= M_3(W,W,W) = \sum_{k=1}^K \pi_k(W^\top \mu_k) \otimes (W^\top \mu_k) \otimes (W^\top \mu_k)  = \sum_{k=1}^K \frac{1}{\sqrt{\pi_k}} \tilde{\mu}_k \otimes  \tilde{\mu}_k \otimes \tilde{\mu}_k
\numberthis
\end{align*}

Upon canonical decomposition of  $\tilde{M_3}$, if the eigenvalues and eigenvectors are  $\{ \lambda_k\}_{k=1}^K $ and $\{u_k\}_{k=1}^K$ respectively, then $\lambda_k = \sfrac{1}{\sqrt{\pi_k}} $. i.e., $\pi_k = \lambda_k^{-2}$, and,
\begin{equation}
u_k = \tilde{\mu}_k = \sqrt{\pi_k}W^\top\mu_k= \frac{1}{ \lambda_k}W^\top\mu_k
\end{equation}

The $\mu_k $s can be recovered as $\mu_k = \lambda_kW^\dagger u_k $, where $W^\dagger$ is the pseudo-inverse of $W^\top$, i.e., $W^\dagger =  W \left(W^\top W\right)^{-1} $. The matrix $O$ can be constructed as $O=\big[\mu_1|\mu_2|\dots|\mu_K \big]$. Since we normalize the columns of $O$ as $O_{vk}=\frac{O_{vk}}{\sum_v{O_{vk}}}$. it is sufficient to compute $\mu_k = W^\dagger u_k$, since $\lambda_k$ will be cancelled during normalization.

It is possible to compute the $\gamma_k$ for $k = 1 \dots K$ through the factorization of second and third order moments of the labels. However, it is not possible to match the topics between $\mu_{1:K}$ and $\gamma_{1:K}$. Therefore, we use the cross moment $M_{2L}$ between the words and the labels. If we multiply the tensor $M_{2L}$ twice by $W$, we get 

\begin{align*}
\tilde{M}_{2L} &=M_{2L}(W,W) \\
& =\sum_{k=1}^K{\pi_k  \gamma_k \otimes (W^\top\mu_k) \otimes (W^\top \mu_k)} \\
& =\sum_{k=1}^K{ \gamma_k \otimes (\sqrt{\pi_k}W^\top\mu_k) \otimes (\sqrt{\pi_k}W^\top \mu_k)} \\
& =\sum_{k=1}^K{\gamma_k \otimes \tilde{\mu}_k \otimes \tilde{\mu}_k}
\numberthis
\end{align*}

If the $k$th eigenvalue of $\tilde{M}_3$ is $u_k$, then

\begin{align*}
 &u_k ^\top M_{2L}(W,W)  u_k = \tilde{\mu}_k ^\top M_{2L}(W,W)  \tilde{\mu}_k = \tilde{\mu}_k^\top \left( \sum_{k=1}^K{\gamma_k \otimes \tilde{\mu}_k \otimes \tilde{\mu}_k} \right)  \tilde{\mu}_k = \gamma_k
\end{align*}

i.e., $\gamma_k$ can be retrieved as $u_k ^\top M_{2L}(W,W)  u_k$, since $\{\tilde{\mu}_k \}_{k=1}^K$ are orthonormal. Thus, we can make sure that $\mu_k$ and $\gamma_k$ will correspond to the same topic $k$ for $k = 1,2 \dots K$.

Therefore,
\begin{align*}
\label{eqn:L}
Q& =\big[ \gamma_1| \gamma_2| \dots |\gamma_K \big] \\
& =\big[ u_1 ^\top M_{2L}(W,W)  u_1| u_2 ^\top M_{2L}(W,W)  u_2 | \dots | u_K ^\top M_{2L}(W,W)  u_K \big] \\
& = \left[u_1^\top | u_2 ^\top | \dots |u_K^\top\right] M_{2L}(W,W) \left [u_1 | u_2 | \dots |u_K \right] \\
& = U^\top M_{2L}(W,W) U\\
& = M_{2L}(WU,WU)
\numberthis
\end{align*}

where, $U = \left [u_1 | u_2 | \dots |u_K \right] $ are all the $K$ eigenvectors of the tensor $\tilde{M}_3$.

\subsection{Label Prediction}

Once we have $O$ and $\pi$, the probability of a document $d $ given $h$ can be expressed as,

\begin{equation}
P\big[d|h=k\big] =  \prod_{v \in \mathcal{W}_d}P\big[v|h=k\big]
\end{equation}
where $\mathcal{W}_d$ is the set of distinct words in the document $d$. 

Then the document probabilities $P\big[ h=k|d \big]$ can be estimated using Bayes Rule. 
\begin{align*}
 P\big[h=k|d\big]  &=\frac{P\big[h=k\big]\prod_{v \in \mathcal{W}_d}P\big[ v|h=k \big]}{\sum_{k=1}^K P\big[h=k\big] \prod_{v \in \mathcal{W}_d}P\big[ v|h=k\big]} =\frac{\pi_k \prod_{v \in \mathcal{W}_d}O_{vk}}{\sum_{k=1}^K \pi_k\prod_{v \in \mathcal{W}_d}O_{vk}}
  \numberthis \label{eqn:personalization}
\end{align*}

Then the probability of a label $l$ for the document can be computed as,

\begin{align*}
P\big[l | d\big] & = \sum_{k=1}^K P\big[l | h=k\big]P\big[h=k\big| d]  = \sum_{k=1}^K Q_{lk}P\big[h=k\big| d] 
\numberthis
\end{align*}

The labels are ranked by the probabilities $P\big[l | d\big]$, and the labels with highest ranks are assigned to the document. If the number of unique words in a test document is $n_d = |\mathcal{W}_d|$, then the prediction step has a complexity of $\Theta\left((n_d+L)K\right)$ to compute the probability for all $L$ labels.

\section{Moment Estimation}
\label{sec:implementation}
The matrices $M_2$ as well as the tensors $M_3$ and $M_{2L}$ are defined on population. We cannot compute the population parameters; all we can do is to estimate them from the sample corpus and compute an error bound for the estimation. We denote the estimated values of $M_2$, $M_3$, $M_{2L}$, $W$, $U$, $O$, $Q$ and $\pi$ from the sample, as $\hat{M}_2$, $\hat{M}_3$, $\hat{M}_{2L}$, $\hat{W}$, $\hat{U}$, $\hat{O}$, $\hat{Q}$ and $\hat{\pi}$ respectively, conforming with the notations used in \cite{MoM}.
 
We create an estimation of the joint probability of the words ($M_2$) by counting the pairwise occurrence of the words in all the documents, and normalizing by the sum. If $X \in \mathbb{R}^{N \times  D}$ is the binary sparse binary matrix representing the data, then the pairwise occurrence matrix can be estimated by $X^\top X$, whose sum of all elements is,
\begin{align*}
\sum_v \sum_v X^\top X & = \sum_v \sum_v \sum_{i=1}^N x_i^\top x_i =\sum_{i=1}^N  \sum_v \sum_v x_i^\top x_i =\sum_{i=1}^N nnz(x_i)^2 \\
\end{align*}
where $x_i \in \mathbb{R}^D$ is the row of $X$ corresponding to the $i$th document, and $nnz(x_i)$ is the total number of the words in that document. Therefore, $M_2$ can be estimated as,

\begin{equation}
\hat{M}_2 = \frac{1}{\sum_{i=1}^N{nnz(x_i)^2}}X^\top X
\end{equation}

Computing $\hat{M}_2$ takes a single pass through the entire corpus.

Similarly, the triple-wise occurrence tensor can be estimated as $X \otimes X \otimes X$, and the sum of all of the elements of the tensor is $\sum_v \sum_v \sum_v X \otimes X \otimes X = \sum_{i=1}^N nnz(x_i)^3$. Therefore, $M_3$ can be estimated as,
\begin{equation}
\hat{M}_3 = \frac{1}{\sum_{i=1}^N{nnz(x_i)^3}}X \otimes X \otimes X
\end{equation}

 \begin{algorithm*}[tb]
\addtolength\linewidth{-2ex}
   \caption{Method of Moments for Parameter Extraction}
   \label{alg:mom}
\begin{algorithmic}
   \STATE {\bfseries Input:} Sparse Data $ X \in \mathbb{R}^{N \times D}$, Label $ Y \in \mathbb{R}^{N \times L}$ and $K \in \mathbb{Z}^+$ \\{\bfseries 
   Output:} $ P\big[v|h\big]$, $ P\big[l|h\big] $ and $\pi$

\begin{enumerate}
\item Estimate \begin{equation}
                \hat{M}_2 = \frac{1}{ \sum_{i=1}^N {nnz(x_i)^2}} X^\top X   \tag*{(pass \#1)}
            \end{equation}
            
  \item Compute maximum eigenvalues $K$ of $\hat{M}_2$ as $\{ \nu_k \}_{k=1}^K$, and corresponding eigenvectors as  $\{ \omega_k \}_{k=1}^K$. Define $\Omega=\big[ \omega_1|\omega_2| \dots |\omega_K \big]$, and  $\Sigma = diag\left(\nu_1,\nu_2,\dots,\nu_K\right)$
  % as $ \hat{M_2} = \mathcal{U}_{D\times K} \Sigma_{K\times K} \mathcal{U}_{D\times K}^\top$.
  \item Estimate the whitening matrix $ \hat{W} = \Omega{\Sigma}^{-1/2}$ so that $ \hat{W}^\top \hat{M}_2 \hat{W} =I_{K\times K}$
  \item Estimate
       \begin{equation*} 
         \hat{\tilde{M}}_3= \frac{1}{\sum_{i =1}^N {nnz(x_i)^3}} X\hat{W} \otimes X\hat{W}\otimes X\hat{W}     \tag*{(pass \#2)}
    \end{equation*}

  \item Compute eigenvalues $\{ \lambda_k\}_{k=1}^K $ and eigenvectors $\{u_k\}_{k=1}^K $ of $\hat{\tilde{M}}_3$. Assign $\hat{U} = [ u_1 | u_2 \dots |u_K]$.
  \item Estimate the columns of $O$ as $\hat{\mu}_k =  \hat{W}^\dagger u_k $ and $\hat{\pi}_k= \lambda_k^{-2}$, 
  $\forall k \in 1,2 \dots K$ 
  \item Assign $\hat{O} = [\hat{\bar{\mu}}_1|\hat{\bar{\mu}}_2|\dots|\hat{\bar{\mu}}_K] $ \& $\hat{\pi}=[\hat{\pi}_1, \hat{\pi}_2 \dots \hat{\pi}_K]^\top $
  \item Estimate 
      \begin{equation*}
        \hat{Q} = \frac{1}{\sum_{i =1}^N {nnz(x_i)^2nnz(y_i)}} Y \otimes X\hat{W}\hat{U}  \otimes X\hat{W}\hat{U}  \tag*{(pass \#3)}
     \end{equation*}

  \item Estimate $P\big[ v|h=k \big]= \frac{\hat{O}_{vk}}{\sum_{v} {\hat{O}_{vk}}}, \forall k \in 1\dots K, v \in v_1\dots v_D$ \\
    $P\big[ l|h=k \big]= \frac{\hat{Q}_{lk}}{\sum_{l} {\hat{Q}_{lk}}}, \forall k \in 1\dots K, l \in l_1\dots l_L$
        
\end{enumerate}
\end{algorithmic}
\end{algorithm*}

The dimensions of $M_2$ and $M_3$ are $D^2$ and $D^3$ respectively, but in practice, these quantities are extremely sparse. $M_2$ has a total number of elements $\mathcal{O} \left( \sum_{i=1}^N nnz(x_i)^2 \right)$, with the worst case occurring when no two documents has any word in common, and all the pairwise counts are $1$. The whitening of $M_2$ is carried out through extracting the $K$ maximum eigenvalues and corresponding eigenvectors. This step is the bottleneck of the algorithm. We use the eigs function in Matlab for computing the eigenvalues, which uses Arnoldi's iterations, and has a complexity $\mathcal{O}\left( ( \sum_{i=1}^N nnz(x_i)^2 )K\right)$, since the sum of all non-zero entries in $M_2$ is $\mathcal{O} \left( \sum_{i=1}^N nnz(x_i)^2 \right)$ \cite{arnoldi}.

As for $M_3$, we do not need to explicitly compute it. Since $\tilde{M_3}=M_3(W,W,W)$, once we have estimated the whitening matrix, say $\hat{W}$, by whitening $\hat{M}_2$, we can estimate $\tilde{M_3}$ right away as,
\begin{equation}
\hat{\tilde{M}}_3=\frac{1}{\sum_{i =1}^N {nnz(x_i)^3}}X\hat{W} \otimes X\hat{W} \otimes X\hat{W}
 \end{equation}
 Computing $\hat{\tilde{M}}_3$ takes a second pass through the entire dataset, and has a complexity of $\mathcal{O}(NK^3)$.
 
 Similarly, if $Y \in \mathbb{R}^{N \times L}$ represents the labels for $N$ documents, then 
 \begin{equation}
 \sum_l \sum_v \sum_v Y \otimes X \otimes X = \sum_{i=1}^N nnz(y_i)nnz(x_i)^2
 \end{equation}
  where $y_i$ is the label vector of $i$th document, and $nnz(y_i)$ is the total number of labels assigned to that document. 
  
 Therefore, $M_{2L}$ can be estimated as,
 \begin{equation}
\hat{M}_{2L}=\frac{1}{\sum_{i =1}^N {nnz(x_i)^2nnz(y_i)}} Y \otimes X \otimes X
 \end{equation}
 
We do not need to compute $\hat{M}_{2L}$ either. Since $Q = M_{2L}(WU,WU)$ from Equation \ref{eqn:L}, once we obtain the eigenvectors $\hat{U}$ of $\hat{\tilde{M}}_3$, we can estimate $Q$ right away as,
 
\begin{equation}
\hat{Q} = \frac{1}{\sum_{i =1}^N {nnz(x_i)^2nnz(y_i)}} Y \otimes X\hat{W}\hat{U}  \otimes X\hat{W}\hat{U}
 \end{equation}
 
 This step has a complexity of $\mathcal{O}(K^2 \sum_{i =1}^N nnz(y_i))$. The entire algorithm is outlined as Algorithm \ref{alg:mom}. The eigenvalue decomposition of $\tilde{M}_3$ has a complexity of $\mathcal{O}\left(K^4\log(1/\epsilon) \right)$ to compute each of the $K$ eigenvectors up to an accuracy of $\epsilon$ \cite{TTB_SSHOPM}. The overall complexity is 
 \begin{align*}
 \mathcal{O}{\left(  \Big(\sum_{i=1}^N nnz(x_i)^2 \Big)K + K^2 \sum_{i =1}^N nnz(y_i) + NK^3 + K^4\log(1/\epsilon)\right)}
 \end{align*}
 We used the Tensor Toolbox \cite{TTB_Software} for tensor decomposition. Once the matrix $\hat{O}$  and  $\hat{\pi}$ are estimated, it requires one more pass through the entire dataset to compute $\hat{Q}$, resulting in a total of three passes to extract all parameters. The label prediction step has a complexity of $\Theta\left((n_d+L)K\right)$ for a document with distinct number of words $n_d$.

 \section{Convergence Bound on Parameters}
\newtheorem{theorem}{Theorem}
\label{thm:bound}
\begin{theorem}
\label{thm:bound}
Let us assume that we draw $N$ i.i.d samples $x_1, x_2 \dots x_N$ with labels $y_1,y_2 \dots y_N$ using the generative process in Equation \ref{1}. Let us define $\varepsilon_1 = \left( 1 + \sqrt{\frac{\log (1 / \delta)}{2}}\right)$, and  $\varepsilon_2 = \left( 1 + \sqrt{\frac{\log (2 / \delta)}{2}}\right)$ for some $\delta \in (0,1)$. Then, if the number of samples $N \ge \max ( n_1, n_2, n_3)$, where

\begin{itemize}
\item  $n_1 =  c_2\left( \log{K} + \log{\log{\left(\frac{K}{c_1}   \cdot \sqrt{\frac{\pi_{max}}{\pi_{min}}} \right)}} \right)$
\item  $n_2 =  \Omega\left( \left(  \frac{\varepsilon_1}{\tilde{d}_{2s}\sigma_K(M_2)} \right)^2 \right)$
\item  $n_3 =  \Omega\left( K^2 \left( \frac{10}{\tilde{d}_{2s}\sigma_K(M_2)^{5/2}}  + \frac{2\sqrt{2}}{\tilde{d}_{3s}\sigma_K(M_2)^{3/2}} \right)^2 \varepsilon_1^2 \right) $
\end{itemize}

for some constants $c_1$ and $c_2$, and we run Algorithm \ref{alg:mom} on these $N$ samples, then the following bounds on the estimated parameters hold with probability at least $1 - \delta$, 

\begin{align*}
 ||\mu_k-\hat{\mu}_k|| &   \le   \left( \frac{160\sqrt{\sigma_1(M_2)}}{\tilde{d}_{2s}\sigma_K(M_2)^{5/2}}  + \frac{32\sqrt{2\sigma_1(M_2)}}{\tilde{d}_{3s}\sigma_K(M_2)^{3/2}} + \frac{4 \sqrt{\sigma_1(M_2)} }{\tilde{d}_{2s}\sigma_K \left( M_2 \right)} \right) \frac{\varepsilon_1}{\sqrt{N}}\\
 ||\gamma_k-\hat{\gamma}_k|| &   \le  \left( \frac{160}{\tilde{d}_{2s}\sigma_K(M_2)^{7/2}}  + \frac{32\sqrt{2}}{\tilde{d}_{3s}\sigma_K(M_2)^{5/2}} +  \frac{ 2+2\sqrt{2}}{\tilde{d}_{2s}\sigma_K(M_2)^2} \right) \frac{2\varepsilon_1}{\sqrt{N}}  + \frac{8\varepsilon_2}{\tilde{d}_{ls}\sigma_K(M_2)\sqrt{N}}\\
 |\pi_k-\hat{\pi}_k| & \le  \left( \frac{200}{\sigma_K(M_2)^{5/2}}+ \frac{40\sqrt{2}}{\sigma_K(M_2)^{3/2}} \right)\frac{\varepsilon_1}{\tilde{d}_{3s}\sqrt{N}} 
\end{align*}

 where $\sigma_1(M_2) \dots \sigma_K(M_2)$ are the K largest eigenvalues of the pairwise probability matrix $M_2$, $\tilde{d}_{2s} = \frac{1}{N} \sum_{i=1}^N nnz(x_i)^2$, $\tilde{d}_{3s} =\frac{1}{N} \sum_{i=1}^N nnz(x_i)^3$ and $\tilde{d}_{ls} = \frac{1}{N} \sum_{i=1}^N nnz(x_i)^2nnz(y_i)$. 
\end{theorem}
The proof is included in the appendix.

 \begin{table*}[t]
 \scriptsize
\begin{center}
\begin{tabular}{|>{\centering\arraybackslash}p{1cm}|>{\centering\arraybackslash}p{1.1cm}|>{\centering\arraybackslash}p{1.1cm}|>{\centering\arraybackslash}p{1.1cm}|>{\centering\arraybackslash}p{1cm}|>{\centering\arraybackslash}p{1.4cm}|>{\centering\arraybackslash}p{1.4cm}|>{\centering\arraybackslash}p{1.1cm}|>{\centering\arraybackslash}p{1.1cm}|} \hline
%\midrule
Dataset  & Feature Dimension & Labels Dimension &  $\#$ of Train Points   &  $\#$ of Test Points  &  Average $\#$ of Non-Zero Fts. & Median $\#$ of Non-Zero Fts. & Average $\#$ of Labels & Median $\#$ of Labels \\  \hline
Bibtex & 1836 & 159 & 4880 & 2515 & 68.67 & 69 & 2.4 & 2\\ \hline
Delicious & 500 & 983 & 12920 & 3185 & 18.29 & 6 & 19.02 & 20 \\ \hline
NYTimes & 24670 & 4185 &  14669 & 15989 & 373.91 & 354 &  5.40 &   5 \\ \hline
Wiki-31K & 101938 & 30938 &14146 & 6616 & 669.05  & 513 & 18.64  &  19  \\ \hline
AmazonCat & 203882 & 13330 & 1186239 & 306782 & 71.09 & 45 & 5.04 & 4\\ \hline
WikiLSHTC & 1617899 & 325056 & 1778351 & 587084 & 42.15 &  30 &  3.19 &    2 \\ \hline
\end{tabular}
\end{center}
\caption{Description of the Datasets (Fts. stands for Features)}
\label{table:desc}	
\end{table*}

 \begin{table*}[tb]
 \footnotesize
\begin{center}
\begin{tabular}{>{\centering\arraybackslash}p{2cm}|>{\centering\arraybackslash}p{5.27cm}|>{\centering\arraybackslash}p{5.27cm}}
\small{True Labels}  & \small{LEML} & \small{MoM}  \\ \hline
"airlines and airplanes", "hijacking", "terrorism" & \textbf{"airlines and airplanes"} (0.34), \textbf{"terrorism"} (0.30), "united states international relations" (0.27), "elections" (0.22), "armament, defense and military forces" (0.18), "internationalrelations" (0.18), "bombs and explosives" (0.15), "murders and attempted murders" (0.13), "biographical information" (0.13), "islam" (0.12) & \textbf{"terrorism"} (0.12), "united states international relations" (0.08), \textbf{"airlines and airplanes"} (0.07), "world trade center (nyc)" (0.07), \textbf{"hijacking"} (0.07), "united states armament and defense" (0.07), "pentagon building" (0.03), "bombs and explosives" (0.03), "islam" (0.02), "missing persons" (0.02) \\ \hline
"armament, defense and military forces", "civil war and guerrilla warfare", "politics and government" & \textbf{"civil war and guerrilla warfare"} (0.62), "united states international relations" (0.39), "united states armament and defense" (0.23), \textbf{"armament, defense and military forces"} (0.23), "internationalrelations" (0.17), "oil (petroleum) and gasoline" (0.11), "surveys and series" (0.10), "military action" (0.09), "foreign aid" (0.08), "independence movements" (0.08) & "united states international relations" (0.09), \textbf{"civil war and guerrilla warfare"} (0.09), "united states armament and defense" (0.06), \textbf{"politics and government"} (0.04), \textbf{"armament, defense and military forces"} (0.03), "internationalrelations" (0.02), "immigration and refugees" (0.02), "foreign aid" (0.02), "terrorism" (0.02), "economic conditions and trends" (0.02) \\ \hline
\bottomrule	 
\end{tabular}
\end{center}
\caption{Examples of Label Prediction from NYTimes Dataset. The numbers in parenthesis indicate the predicted scores. The scores of LEML and MoM are not directly comparable, since the score of LEML can be negative, whereas the score of MoM lies within $[0,1]$}
\label{table:Demo}	
\end{table*}

\section{Experimental Results}

We used six datasets for our methods, as described in table \ref{table:desc}. The datasets range from small datasets like Bibtex with $4880$ training instances with $159$ labels to large datasets like WikiLSHTC with around $1.7$M training instances with $325$K labels. Since LEML is shown to outperform WSABIE and other benchmark algorithms on various small and large-scale datasets in \cite{LEML}, we benchmark the performance of our method against LEML. Also both LEML and MoM has similar model complexity due to similar number  $\left(\Theta\left((L+D)K\right)\right)$ of parameters for the same latent state dimensionality $K$. For LEML, we ran ten iterations for the smaller datasets (Bibtex and Delicious) and five iterations for the larger datasets, since the authors of LEML chose a similar number of iterations for their experiments in \cite{LEML}. We measured AUC (of Receiver Operating Characteristics (ROC)) against $K$. AUC is a versatile measure, and is used to evaluate the performance of classification as well as prediction algorithms \cite{BPR}. Also, it is shown that there exists a one-to-one relation between AUC and Precision-Recall curve in \cite{AUC}, i.e., a classifier with higher AUC will also achieve better Precision and Recall. We carried out our experiments on Unix Platform on a single machine with Intel i5 Processor (2.4GHz) and 16GB memory, and no multi-threading or any other performance enhancement method is used in the code. For AmazonCat and WikiLSHTC datasets, we ran LEML on an i2.4xlarge instance of Amazon EC2 with 122 GB of memory, since LEML needs significantly larger memory for these two datasets (Figure 2).

We computed AUC for every test documents and performed a macro-averaging across the documents, and repeated the experiments for $K = \{ 50, 75, 100, 125, 150\}$ (Figure 2). Both LEML and Method of Moments perform very similarly, but the memory footprint (Figure 2) of MoM is significantly less than LEML. MoM takes longer to finish for the smaller datasets like Bibtex or Delicious since tensor factorization takes much more time compared to the LEML iterations. However, for the larger datasets, each iteration of LEML becomes extremely costly, and MoM takes a fraction of the time taken by LEML. For WikiLSHTC dataset, LEML takes more than two days to finish, while MoM finished within a few hours. The runtime as well as speed-up is shown in Table \ref{table:Time} for $K=100$. Due to the large discrepancy between the runtime of LEML and MoM for the larger datasets, we do not give a detailed plot of runtime vs. $K$.

\begin{figure*}[t]
    \label{fig:result}
    \centering
        \begin{subfigure}[b]{0.49\textwidth}
            \centering
	{\includegraphics[ scale=0.17]{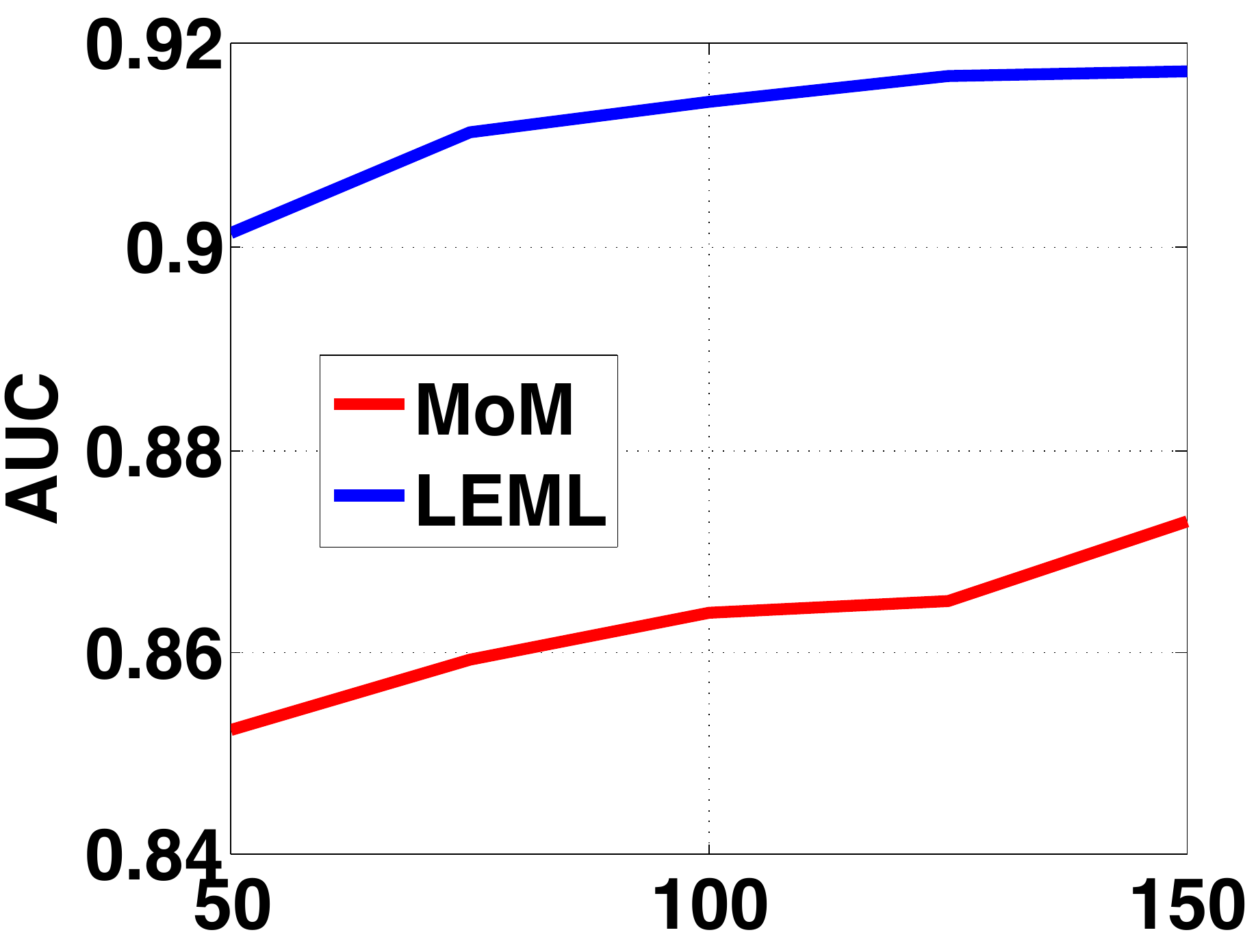}}
      	~
	{\includegraphics[ scale=0.17]{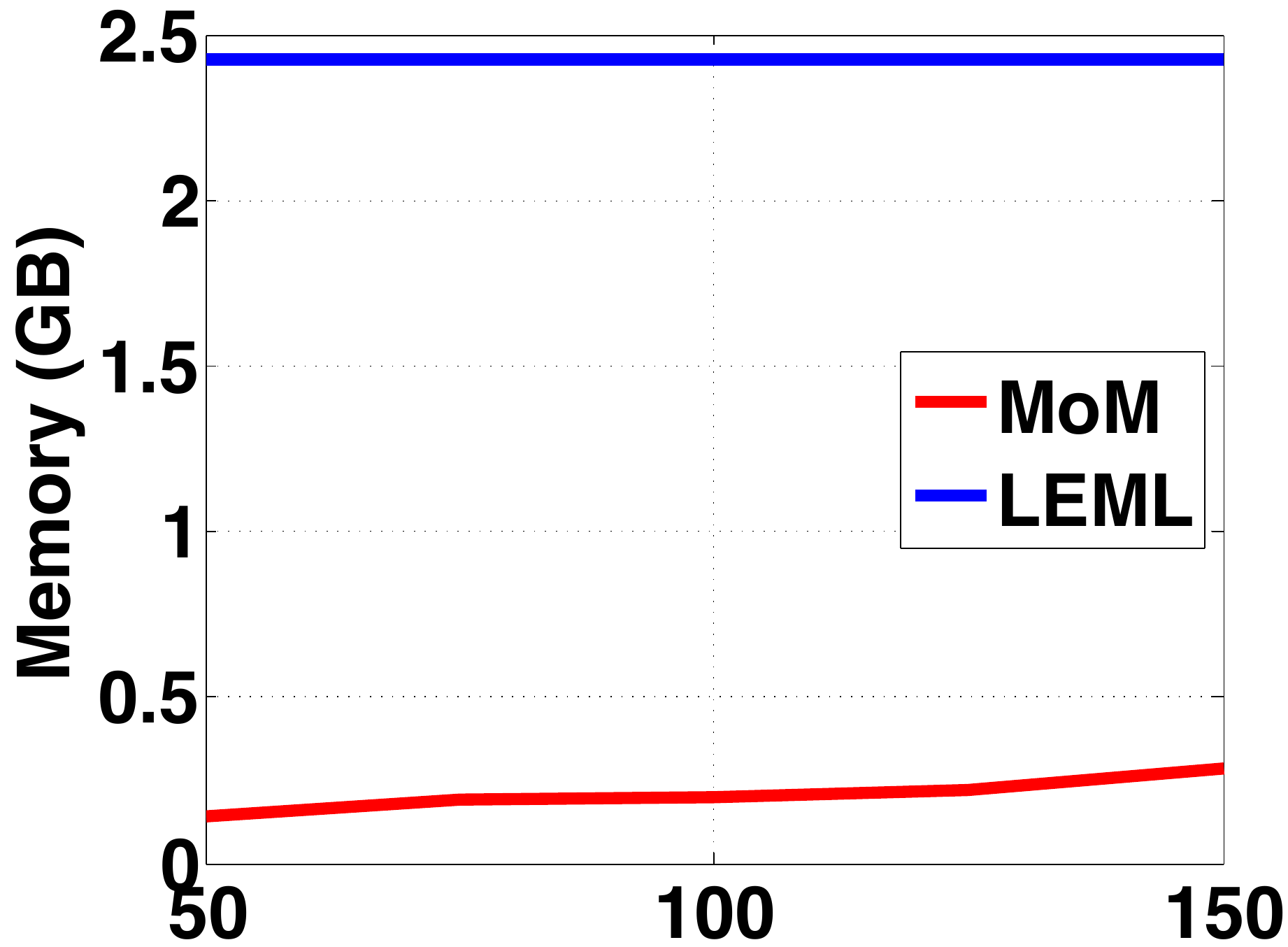}}  
	
	\caption{AUC and Memory (GB) of Bibtex Dataset}

    \end{subfigure}
    ~
        \begin{subfigure}[b]{0.49\textwidth}
            \centering
	{\includegraphics[ scale=0.17]{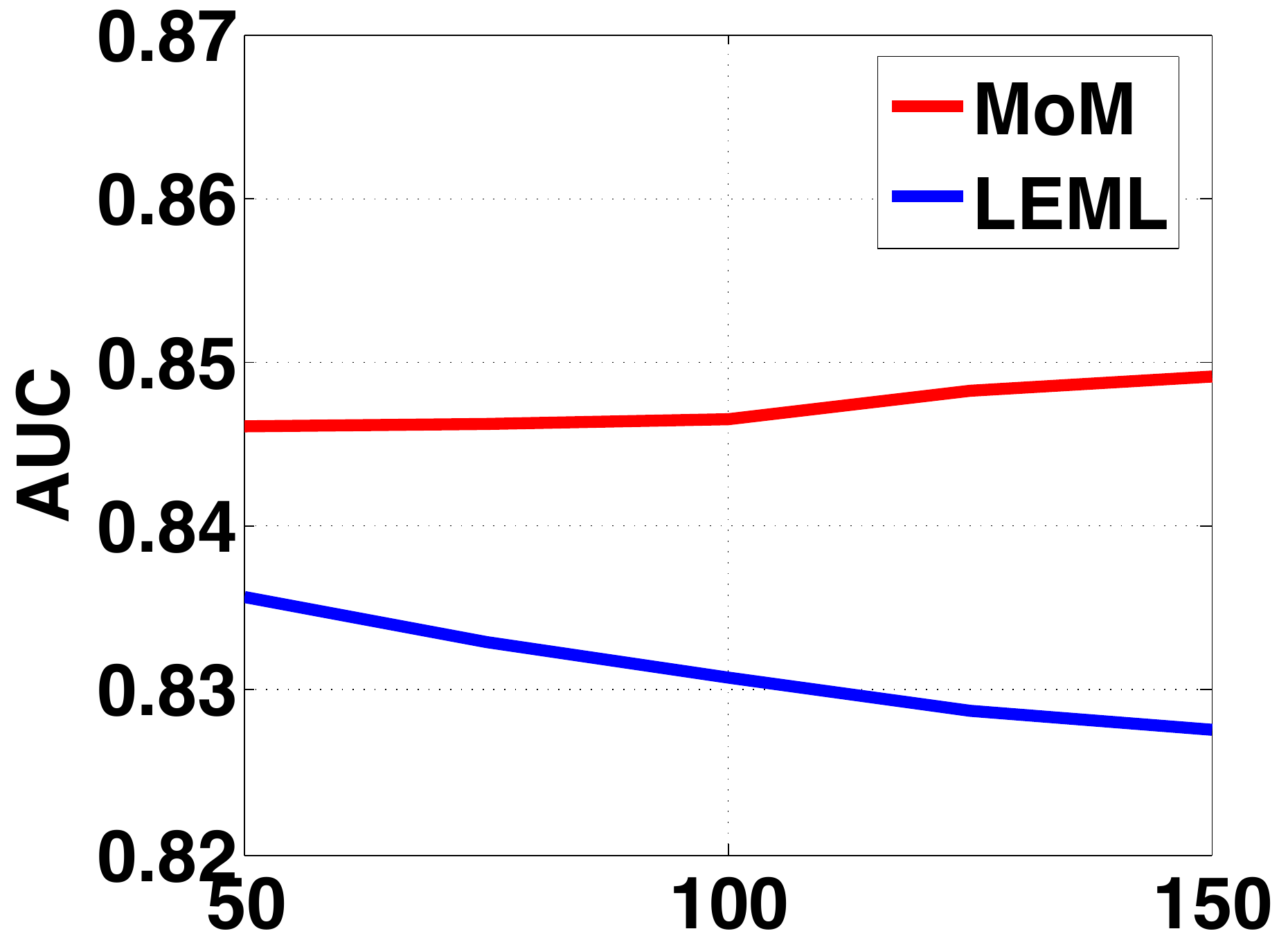}}
      	~
	{\includegraphics[ scale=0.17]{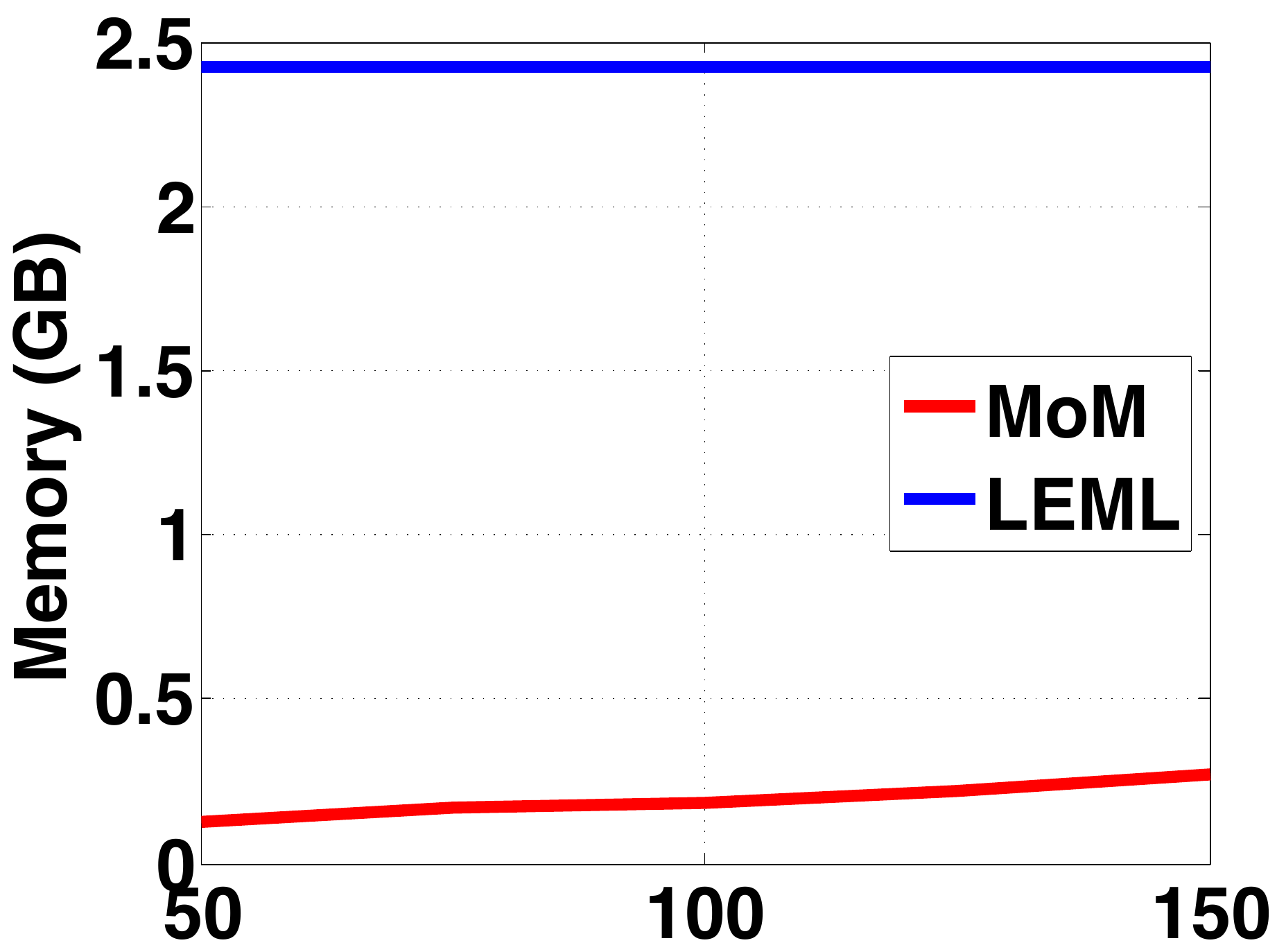}}  
	
	\caption{AUC and Memory (GB) of Delicious Dataset}

    \end{subfigure}
    
    \begin{subfigure}[b]{0.49\textwidth}
            \centering
	{\includegraphics[ scale=0.17]{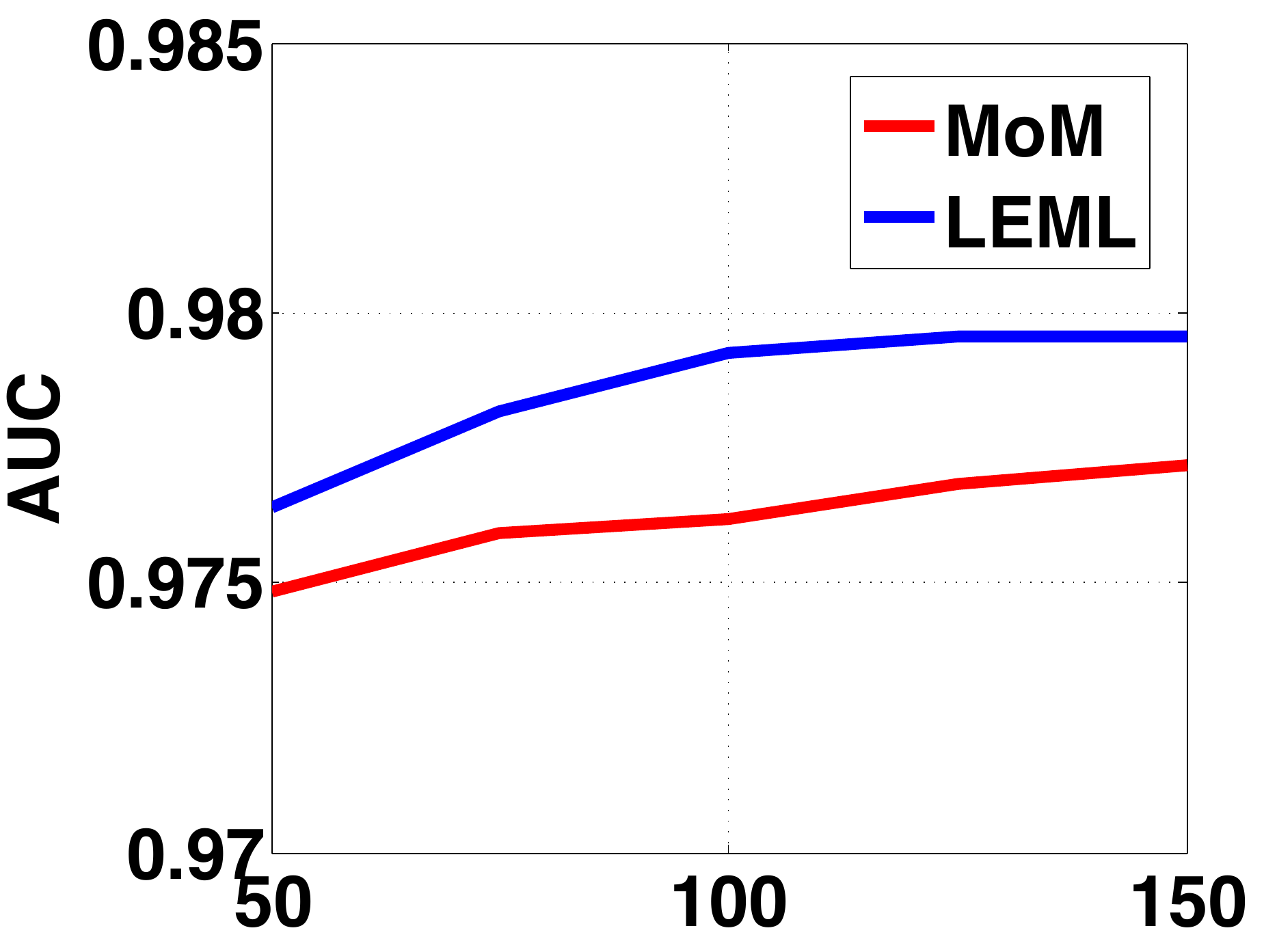}}
      	~
	{\includegraphics[ scale=0.17]{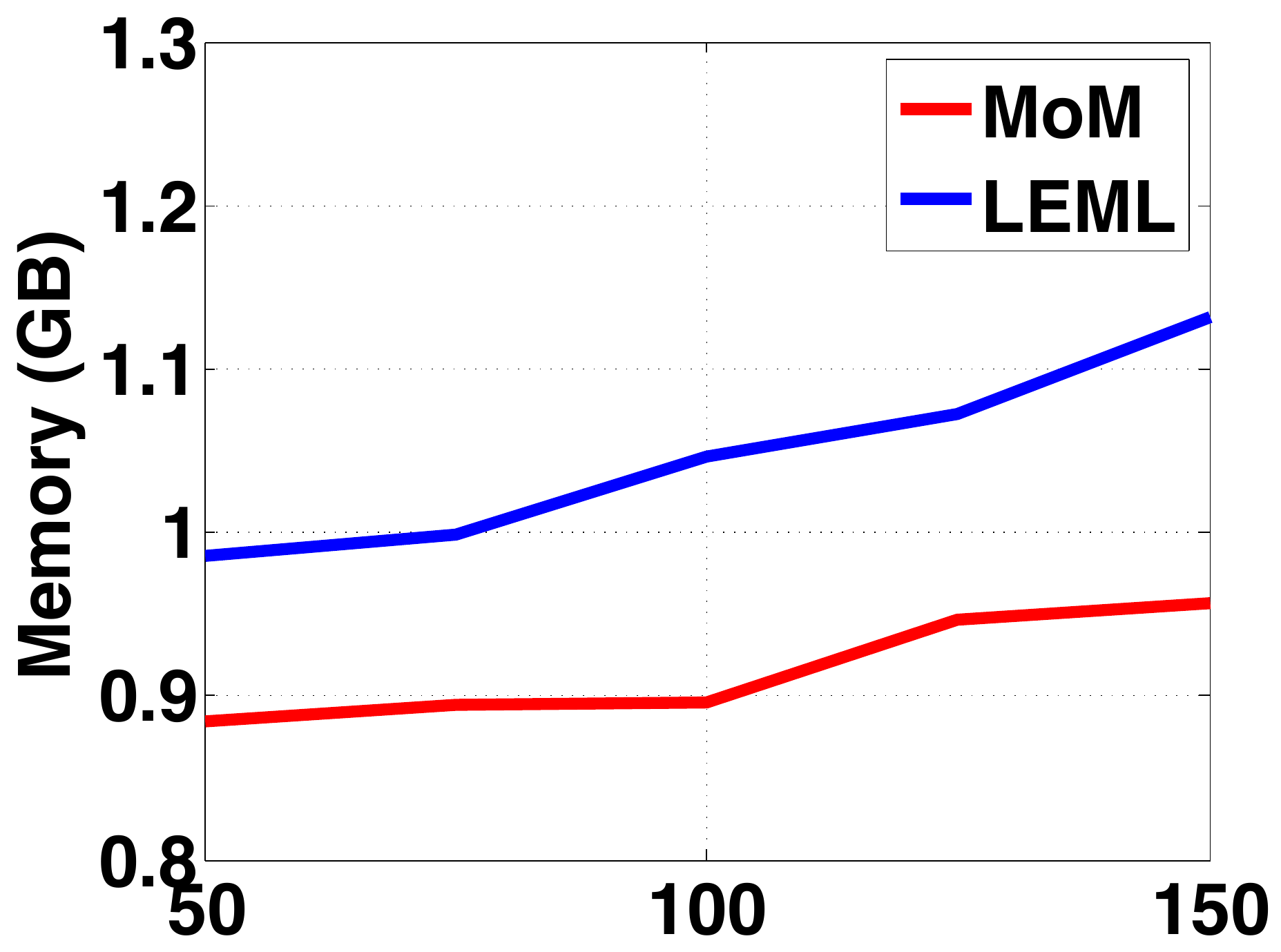}}  
	
	\caption{AUC and Memory (GB) of NYTimes Dataset}

    \end{subfigure}
    ~
     \begin{subfigure}[b]{0.49\textwidth}
            \centering
	{\includegraphics[ scale=0.17]{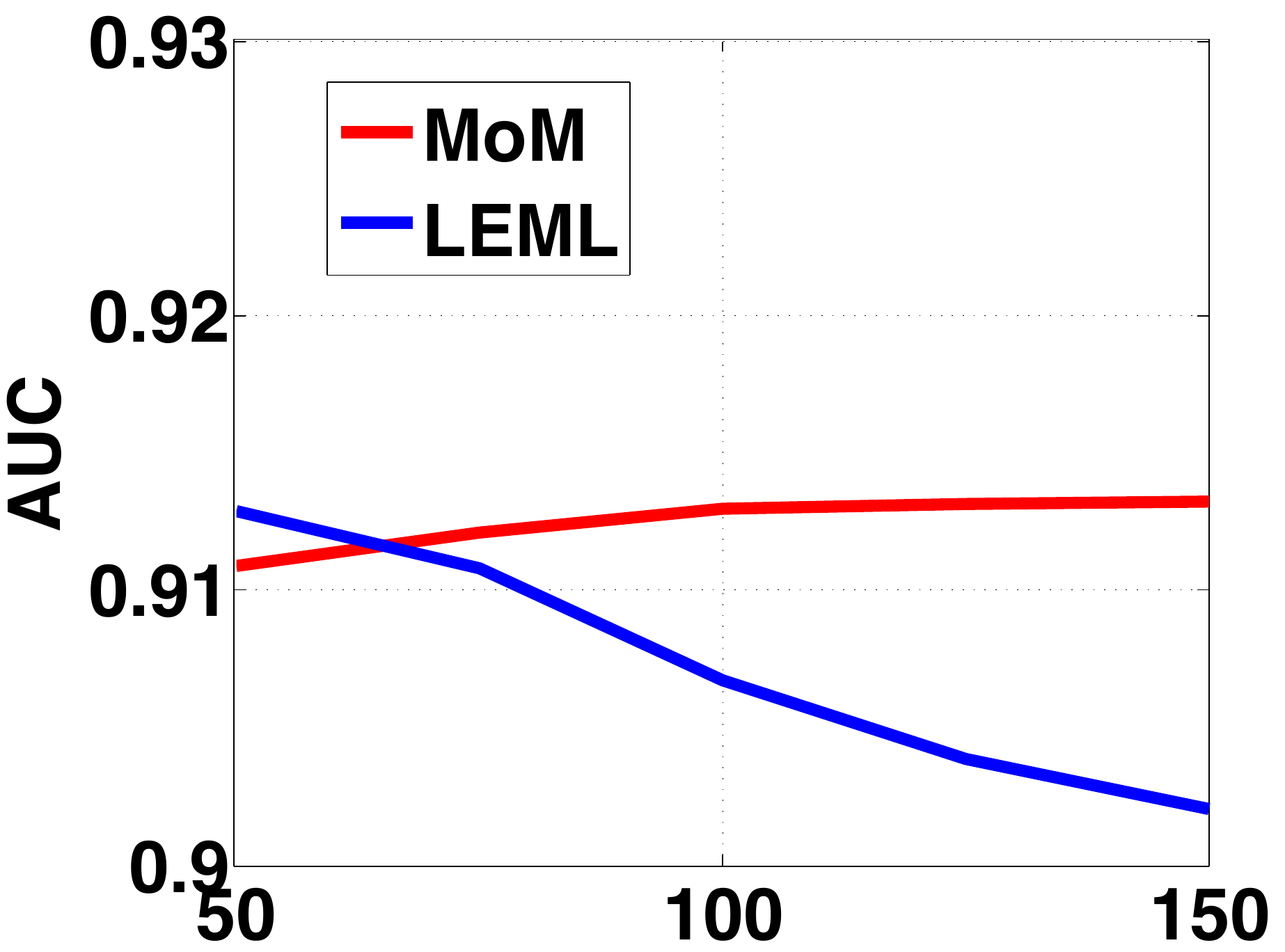}}
      	~
	{\includegraphics[ scale=0.17]{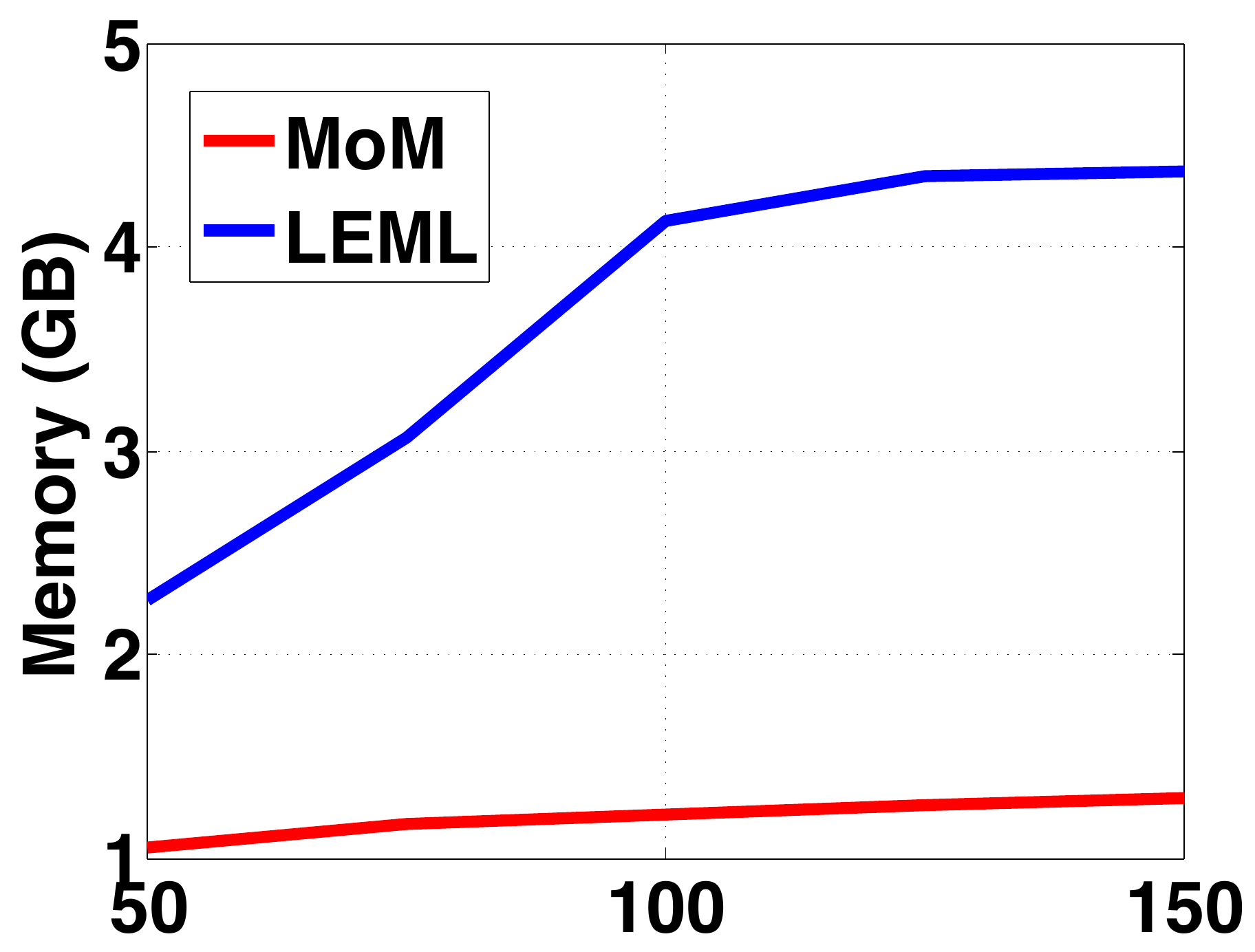}}  
	
	 \caption{AUC and Memory (GB) of Wiki-31K Dataset}
    \end{subfigure}
    
        \centering
    \begin{subfigure}[b]{0.49\textwidth}
            \centering
	{\includegraphics[ scale=0.17]{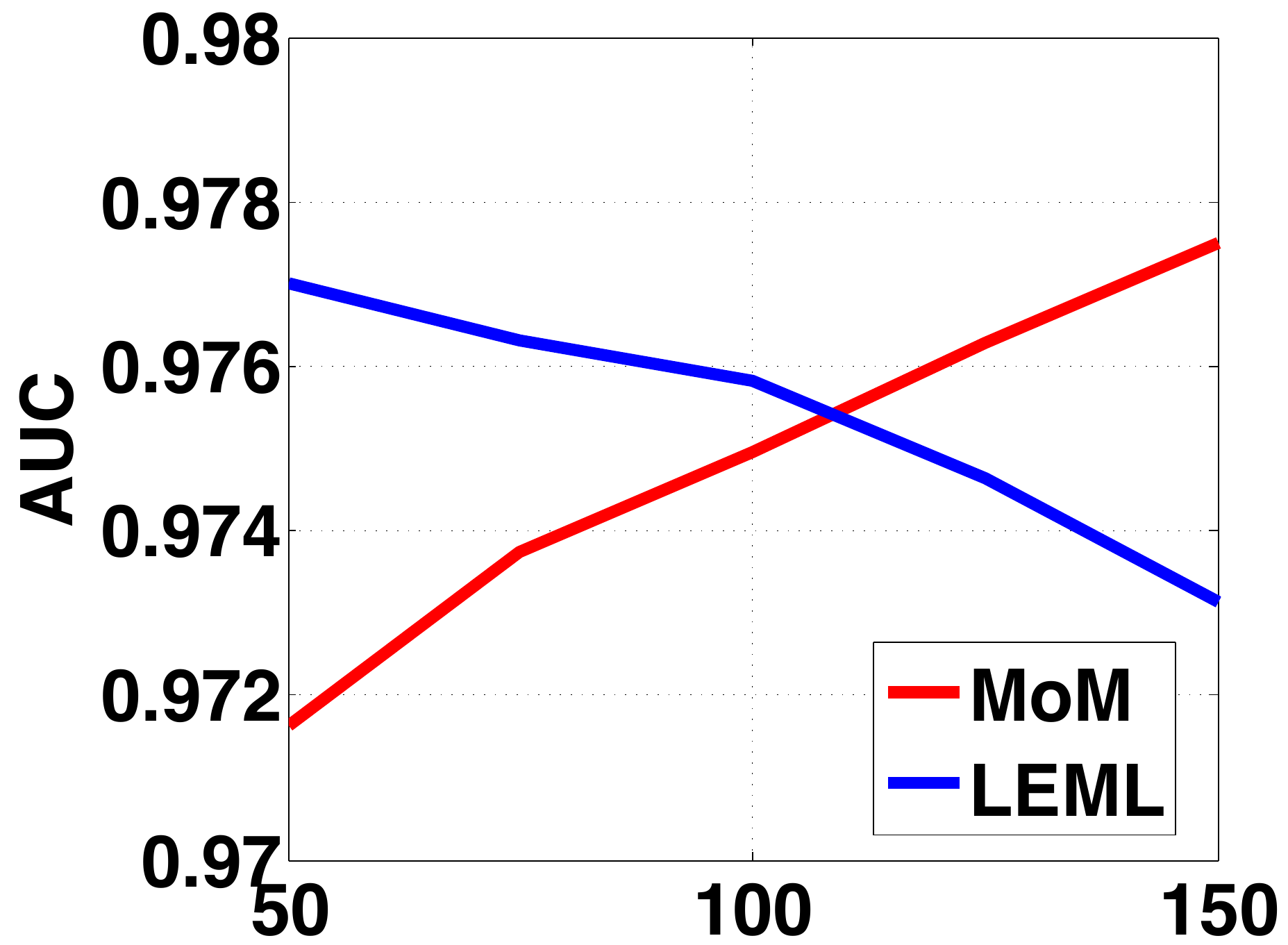}}
      	~
	{\includegraphics[ scale=0.17]{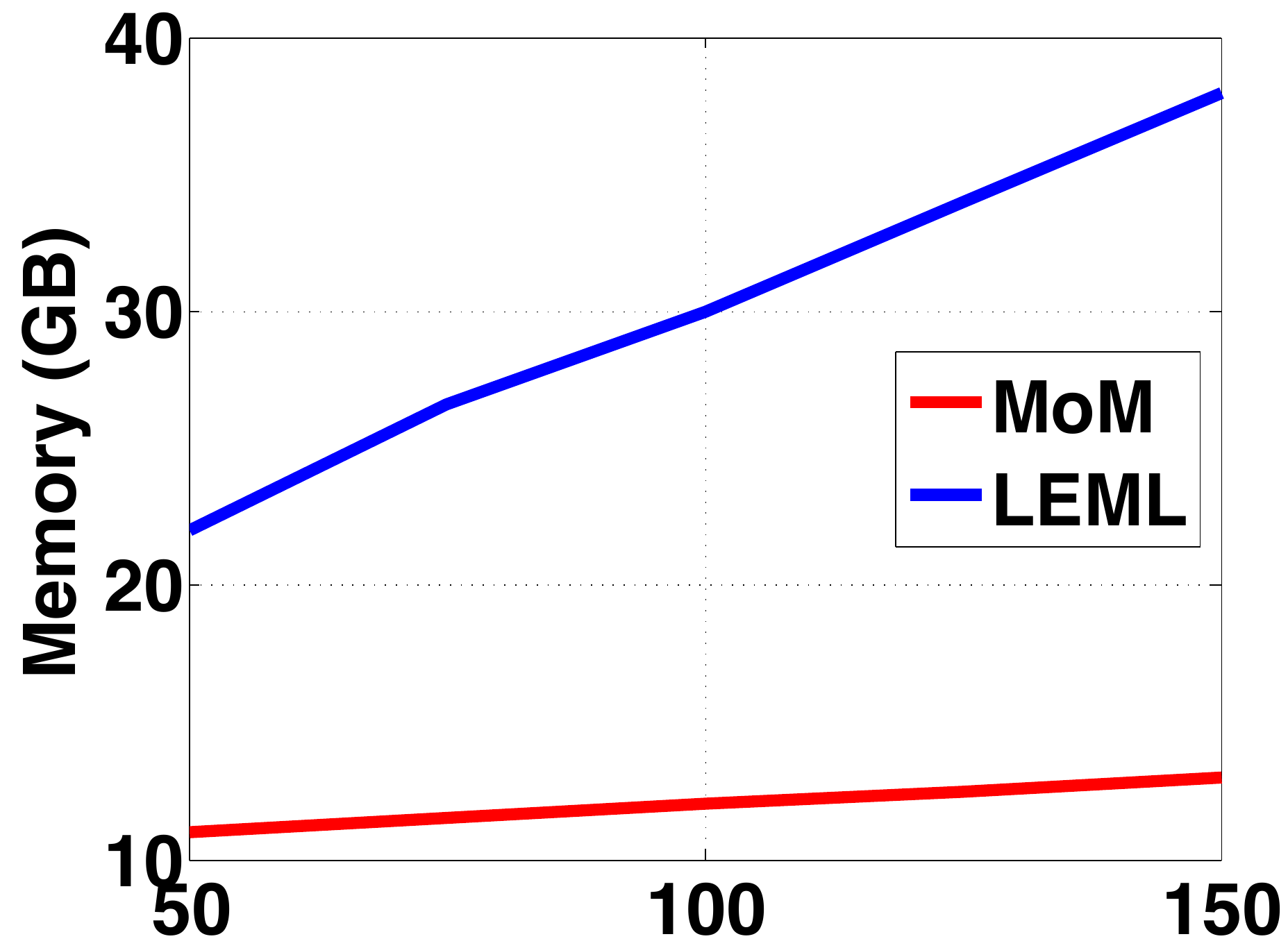}}  
	
	\caption{AUC and Memory (GB) of AmazonCat Dataset}

    \end{subfigure}
    ~
     \begin{subfigure}[b]{0.49\textwidth}
            \centering
	{\includegraphics[ scale=0.17]{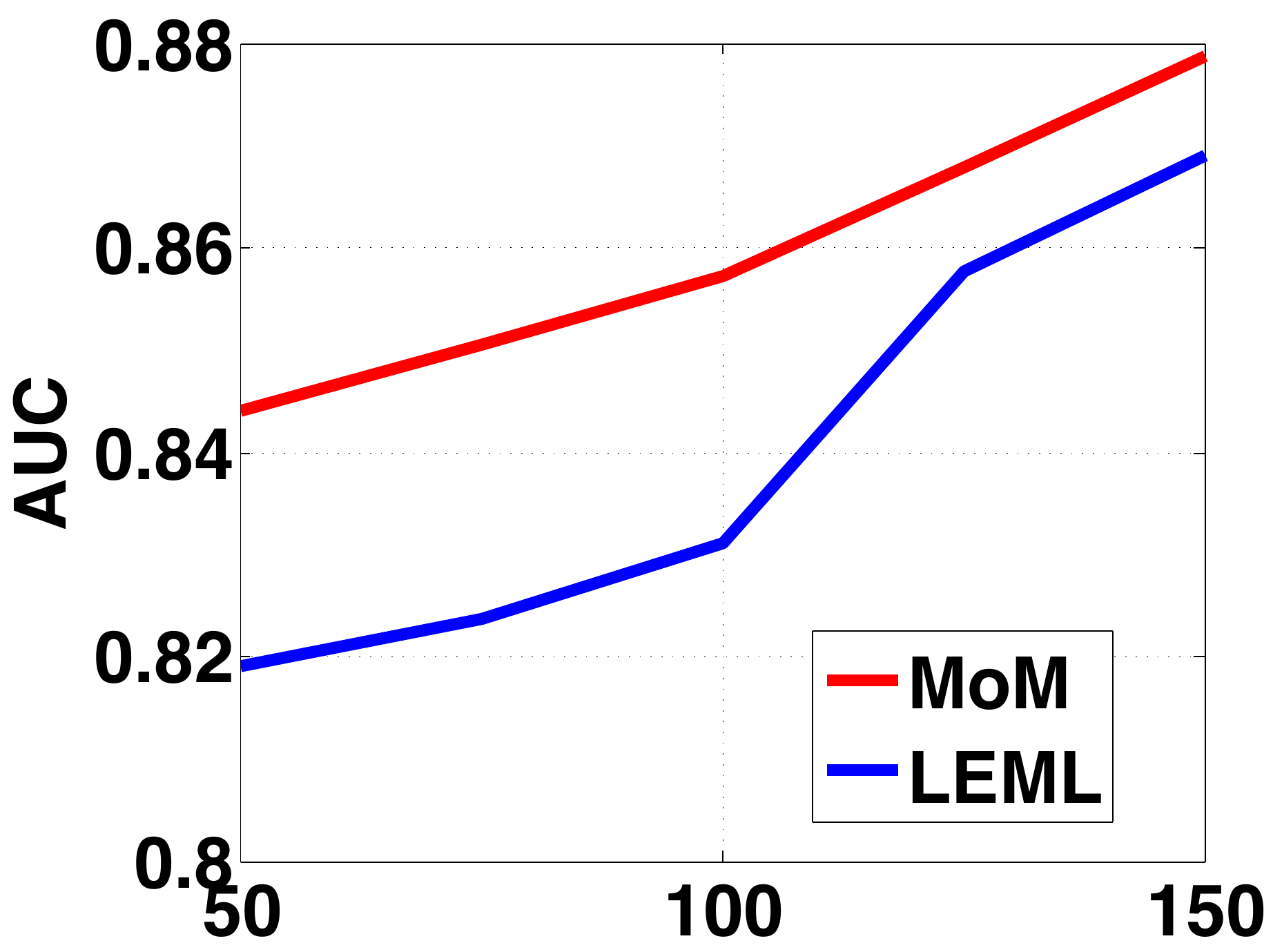}}
      	~
	{\includegraphics[ scale=0.17]{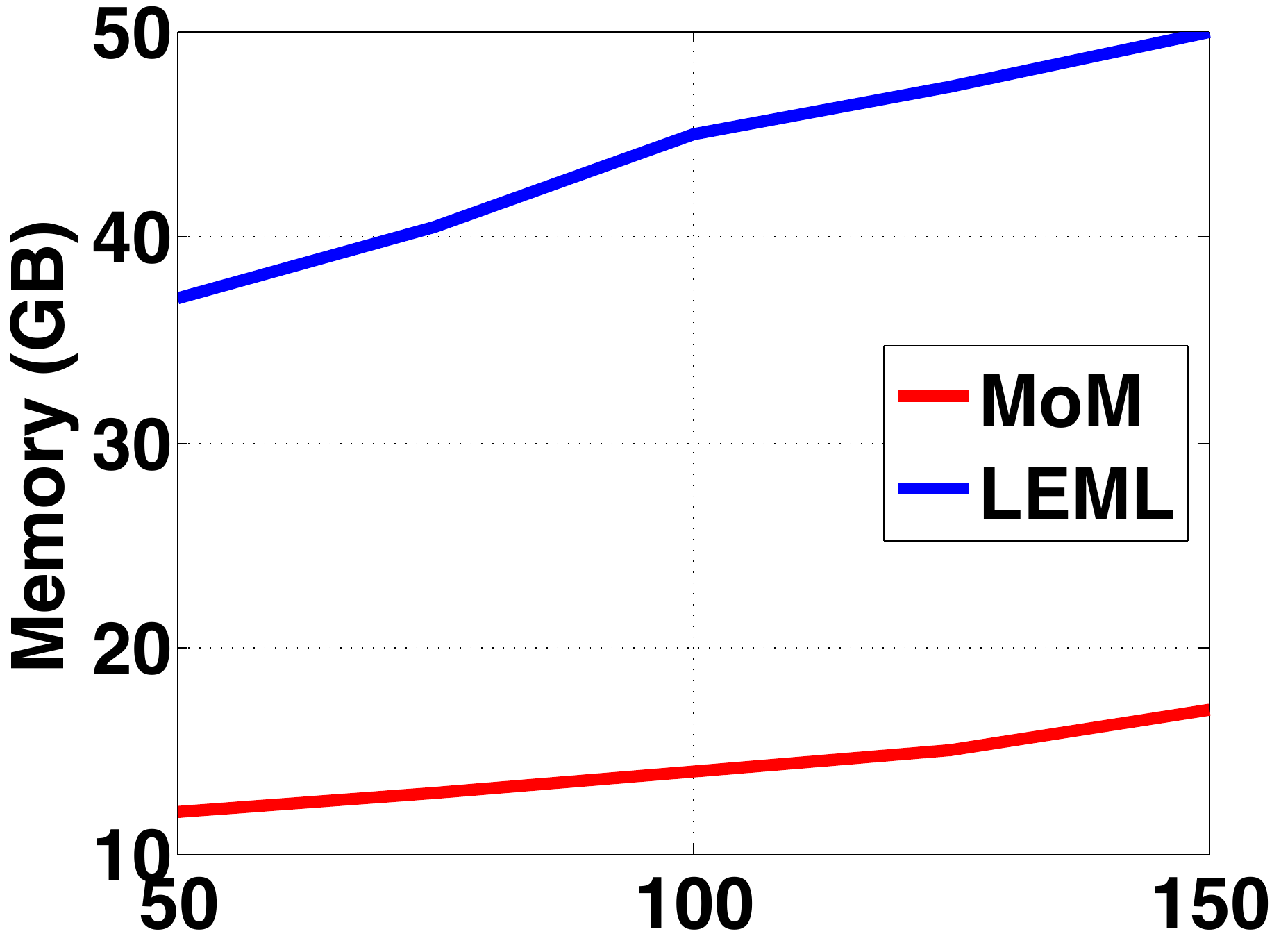}}  
	
	 \caption{AUC and Memory (GB) of WikiLSHTC Dataset}
    \end{subfigure}

    \caption{AUC and Memory vs Latent State Dimensionality$(K)$ for LEML and MoM for Different Datasets. LEML suffers from over-fitting in Delicious, Wiki-31K and AmazonCat Datasets, and the AUC decreases with increasing $K$ }
\end{figure*}

 \begin{table}[h]
 \scriptsize
\begin{center}
\begin{tabular}{>{\centering\arraybackslash}p{1.3cm}>{\centering\arraybackslash}p{1.3cm}>{\centering\arraybackslash}p{1.3cm}>{\centering\arraybackslash}p{1.9cm}}
%\midrule
Dataset  & LEML & MoM  & Speed-up (in $\times$) \\  \toprule
Bibtex & $160 s.$ & $300 s.$ & $0.53$  \\ \midrule
Delicious & $60 s.$ & $150 s.$ & $0.4$\\ \midrule
NYTimes & $1$ hour & $6$ min. & $10$\\ \midrule
Wiki-31K & $3 \text{ hour } 40 \text{ min }$ & $15$ min & $15$  \\ \midrule
AmazonCat & $13 \text{ hour } ^\dagger$  &  $1 \text{ hour }15 \text{ min}$ & $10$  \\ \midrule
WikiLSHTC & $>2$ days$^\dagger$ & $\sim 3$ hours & $16$  \\ \midrule
\bottomrule
\multicolumn{3}{l}{$\dagger$ Runtime on i2.4xlarge instance of Amazon EC2}\\	 
\end{tabular}
\end{center}
\caption{Training Time and Speed-up}
\label{table:Time}	
\end{table}

\section{Conclusion}

Here we propose a method for multi-label prediction for large-scale datasets based on moment factorization. Our method (MoM) gives similar performance in comparison with state-of-art algorithms like LEML while taking a fraction of time and memory for the larger datasets. MoM takes only three passes through the training dataset to extract all the parameters. Since MoM consists of only linear algebraic operations, it is embarrassingly parallel, and can easily be scaled up in any parallel eco-system using linear algebra libraries. In our implementation, we used Matlab's linear algebra library based on LAPACK/ARPACK, although we did not incorporate any parallelization.

Both LEML and MoM have error bound of $\mathcal{O}(1/\sqrt{N})$ on training performance w.r.t. the number of training samples $N$. However, when we compute the AUC on test data, the AUC of LEML decreases with latent dimensionality$(K)$ for some datasets, including the larger dataset of AmazonCat containing more than $1$M training instance. This shows the possibility of over-fitting in LEML. MoM, on the other hand, is not an optimization algorithm, and the parameters are extracted from Moment Factorization rather than optimizing any target function. It is not susceptible to over-fitting, which is evident from its performance. On the other hand, MoM has the requirement $N \ge \Omega(K^2)$ on the number of documents in the training set, and it will not work if $N < \Theta(K^2)$. However, for smaller text corpora where $N < \Theta(K^2)$ hold, 1-vs-all classifiers are usually sufficient to predict the labels. We need dimensionality reduction techniques for large text corpora where 1-vs-all classifiers fail, and MoM provides a very competitive choice for such cases.

 \nocite{TTB_SSHOPM} \nocite{sleec} \nocite{fastxml}

\bibliographystyle{splncs}
\bibliography{ICLR_RecSys}

\appendix
\section{Vector Norms} 

Let the joint probability matrix and the third order probability tensors of the population be $M_2=p(v,v)$ and $M_3=p(v,v,v)$. Let us assume that we select $N$ i.i.d. samples $x_1, \dots x_N$ from the population, and the estimates of joint probability matrix and third order probability tensor are $\hat{M}_2=\hat{p}(v,v)$ and $\hat{M}_3=\hat{p}(v,v,v)$. Let $\varepsilon_{M_2} = ||M_2 - \hat{M}_2||_2$, where $||\cdot||_2$ denotes the second order operator norm of the corresponding matrix or tensor. Let us assume $\varepsilon_{M_2} \le \sigma_K(M_2)/2$, where $\sigma_K$ is the $K$th largest eigenvalue of $M_2$. We will derive the conditions which satisfies this later. 

If $\Sigma=diag(\sigma_1,\sigma_2 \dots \sigma_K)$ are the top-K eigenvalues of $M_2$, and $U$ are the corresponding eigenvectors, then the whitening matrix $W=U\Sigma^{-1/2}$, and $W^\top M_2 W = I_{K \times K}$. Then,

 \begin{align*}
 ||W||_2 &= \sqrt{ \max \text{eig}(W ^\top W)}=\sqrt{ \max \text{eig}(\Sigma^{-1})} =\frac{1}{\sqrt{\sigma_K(M_2)}}
 \end{align*}
 
 Similarly, if $W ^\dagger = W(W^\top W)^{-1}$, then $W ^\dagger = W\Sigma=U\Sigma^{1/2}$. Therefore,
 
 \begin{equation}
 ||W ^\dagger||_2=\sqrt{ \max \text{eig}(\Sigma)}=\sqrt{\sigma_1 (M_2)}
 \end{equation}

Let $\hat{W}$ be the whitening matrix for $\hat{M}_2$, i.e., $\hat{W}^\top \hat{M}_{2} \hat{W} = I_{K \times K}$. Then by Weyl's inequality, \\
$\sigma_k(M_2) - \sigma_k(\hat{M}_2) \le ||M_2 - \hat{M}_2||, \forall k = 1,2 \dots K$. 

Therefore,
\begin{align*}
||\hat{W}||_2^2 &=\frac{1}{\sigma_K (\hat{M}_2)} \le \frac{1}{\sigma_K\left( M_2 \right) - ||M_2 - \hat{M}_2||} \le \frac{2}{\sigma_K\left( M_2 \right)}, \numberthis
\end{align*}

with the asumption $\varepsilon_{M_2} =||M_2 - \hat{M}_2||_2 \le \sigma_K(M_2)/2$.

Also, by Weyl's Theorem,
\begin{align*}  
& ||\hat{W} ^\dagger||_2^2 = \sigma_1(\hat{M}_2) \le   \sigma_1(M_2) + \varepsilon_{M_2} \le 1.5 \sigma_1(M_2) \\
& \implies ||\hat{W} ^\dagger||_2 \le \sqrt{1.5\sigma_1(M_2)} \le 1.5\sqrt{\sigma_1(M_2)} \numberthis
\end{align*}

Let $D$ be the eigenvectors of $\hat{W}M_2 \hat{W}$, and $A$ be the corresponding eigenvalues. Then we can write, $\hat{W}M_2 \hat{W}$=$ADA^\top$.  Then $W = \hat{W}A D^{-1/2} A^\top$ whitens $M_2$, i.e., $W^\top M_2 W = I$. Therefore,

\begin{align*}
||I-D||_2 &= ||I - ADA^\top||_2\\
&=||I-\hat{W}M_2\hat{W}||_2 \\
&=|| \hat{W}\hat{M}_2\hat{W}-\hat{W}M_2\hat{W}||_2 \\
& \le ||\hat{W}||_2^2 ||M_2 - \hat{M_2}|| \\
&\le \frac{2}{\sigma_K \left( M_2 \right)} \varepsilon_{M_2} \numberthis
\end{align*}

\begin{align*}
\varepsilon_{W} &= ||W-\hat{W}||_2\\
&= ||W-WA D^{1/2} A^\top||_2 \\
&= ||W||_2||I-A D^{1/2} A^\top||_2 \\
&=||W||_2|| I-D^{1/2}||_2 \\
&\le ||W||_2 || I-D||_2 \\
& \le \frac{2}{\sigma_K(M_2)^{3/2}} \varepsilon_{M2}
\end{align*}

\begin{align*}
\varepsilon_{W^\dagger} &= ||{W}^\dagger - \hat{W}^\dagger||_2 \\
& = ||\hat{W}^\dagger A D^{1/2}A^\top - \hat{W}^\dagger||_2 \\
& = ||\hat{W}^\dagger||_2 || I - A D^{1/2}A^\top||_2 \\
& \le ||\hat{W}^\dagger||_2 || I - D||_2 \\
&\le \frac{2 \sqrt{\sigma_1(M_2)} }{\sigma_K \left( M_2 \right)} \varepsilon_{M_2} \numberthis
\end{align*}

\section{Tensor Norms}
Let us define the second order operator norm of a tensor $T \in \mathbb{R}^{D' \times D \times D}$ as,
\begin{equation}
||T||_2 = \sup_u \{ ||T(\cdot,u,u)|| : u \in \mathbb{R}^D \& ||u||=1  \}
\end{equation}

\newtheorem{lemma}{Lemma}
\begin{lemma}
\label{lemma:tnorm}
For a tensor $T \in \mathbb{R}^{D' \times D \times D}$, $||T||_2 \le ||T||_F$, where $||T||_F$ is the Frobenius norm defined as, 
\begin{equation}
||T||_F = \sqrt{\sum_{i,j,k} (T_{i,j,k})^2}
\end{equation}
\end{lemma}

\begin{proof}
For any real matrix $A$, $||A||_2 \le ||A||_F$. Let us unfold the tensor $T$ as the collection of $D'$ symmetric matrices of size $D\times D$ as, $T = \{ T_1,T_2 \dots T_{D'} \}$.
Then,
\begin{align*}
T(\cdot,u,u) &= \left[ u^\top T_1 u, u^\top T_2 u, \dots u^\top T_{D'} u \right]  \\
\end{align*}
Therefore,
\begin{align*}
||T||_2 & = \sup_u \{ ||T(\cdot,u,u)|| : u \in \mathbb{R}^D \& ||u||=1  \} \\
& = \sup_u   \{  \left|\left| \left[ u^\top T_1 u, u^\top T_2 u, \dots  u^\top T_{D'} u \right] \right|\right| : u \in \mathbb{R}^D \& ||u||=1  \}  \\ 
& = \left|\left| \left[ \lambda_{\max}(T_1),  \lambda_{\max}(T_2), \dots   \lambda_{\max}(T_{D'}) \right] \right|\right| \\
& = \left|\left| \big[ \left|\left| T_1 \right|\right|_2 ,  \left|\left| T_2 \right|\right|_2 , \dots   \left|\left| T_{D'} \right|\right|_2 \big]  \right|\right| \\
& \le \left|\left| \big[ \left|\left| T_1 \right|\right|_F ,  \left|\left| T_2 \right|\right|_F , \dots   \left|\left| T_{D'} \right|\right|_F \big]  \right|\right| \\
& = ||T||_F 
\numberthis
\end{align*}

Please note that the Lemma holds well when $D'=D$, and the tensor is symmetric.

\end{proof}

Let us define $\varepsilon_{M_3} = ||M_3 - \hat{M}_3||_2$. Then from Appendix B in \cite{chaganty2013spectral},

\begin{align*}
\label{eqn:etw}
&\varepsilon_{tw} \\
&= ||M_3(W,W,W) - \hat{M}_3(\hat{W},\hat{W},\hat{W})||_2 \\
&\le ||M_3||_2 \left( ||\hat{W}||_2^2 + ||\hat{W}||_2||W||_2 + ||W||_2^2 \right)\varepsilon_W + || \hat{W} ||^3\varepsilon_{M_3} \\
& \le ||M_3||_2\frac{(2 + \sqrt{2} + 1)}{\sigma_K(M_2)} \varepsilon_W + \frac{2\sqrt{2}}{\sigma_K(M_2)^{3/2}}\varepsilon_{M_3} \\
& \le ||M_3||_2\frac{(3 + \sqrt{2} )}{\sigma_K(M_2)} \cdot  \frac{2}{\sigma_K(M_2)^{3/2}} \varepsilon_{M2} + \frac{2\sqrt{2}}{\sigma_K(M_2)^{3/2}}\varepsilon_{M_3} \\
& \le \frac{10||M_3||_2}{\sigma_K(M_2)^{5/2}} \cdot \varepsilon_{M2} + \frac{2\sqrt{2}}{\sigma_K(M_2)^{3/2}}\varepsilon_{M_3} \\
& \le \left( \frac{10}{\tilde{d}_{2s}\sigma_K(M_2)^{5/2}}  + \frac{2\sqrt{2}}{\tilde{d}_{3s}\sigma_K(M_2)^{3/2}} \right) \frac{2\varepsilon}{\sqrt{N}} 
\numberthis
\end{align*}

From Lemma \ref{lemma:tnorm}, $||M_3||_2 \le ||M_3||_F$, and because $M_3$ is a tensor with individual elements as probabilities, and the sum of all the elements is $1$, $ ||M_3||_F \le 1$. Therefore, $||M_3||_2 \le 1$.

\begin{lemma}(Robust Power Method from \cite{MoM})
\label{lemma:rpm}
If $\hat{T} = T + E \in \mathbb{R}^{K \times K \times K}$, where $T$ is an symmetric tensor with orthogonal decomposition $ T = \sum_{k=1}^K{\lambda_k u_k \otimes u_k \otimes u_k}$ with each $\lambda_k >0$, and $E$ has operator norm $||E||_2 \le \epsilon$. Let $\lambda_{\min} = \min_{k=1}^K \{\lambda_k\}$ and $\lambda_{\max} = \max_{k=1}^K \{\lambda_k\}$.
Let there exist constants $c_1,c_2$ such that $\epsilon \le c_1\cdot (\lambda_{\min}/K)$, and $ N \ge c_2( \log{K} + \log{\log{(\lambda_{\max}/\epsilon)}})$. Then if Algorithm 1 in \cite{MoM} is called for $K$ times, with $L = poly(K)\log(1/\eta)$ restarts each time for some $\eta \in (0,1)$, then with probability at least $1-\eta$, there exists a permutation $\Pi$ on $[K]$, such that,
 
 \begin{align*}
 &||u_{\Pi(k)}-\hat{u}_k|| \le 8\frac{\epsilon}{\lambda_{\Pi(k)}} \text{,    } |\lambda_k - \lambda_{\Pi(k)}| \le 5\epsilon \text{ } \forall k \in [K] \numberthis
 \end{align*}
  \end{lemma}

\begin{equation}
\text{Since }\quad \lambda_k = \frac{1}{\sqrt{\pi_k}}, \quad \forall k \in [K] \text{, we need,}
\end{equation}

 \begin{align*}
 N & \ge c_2\left( \log{K} + \log{\log{\left(\frac{K\lambda_{\max}}{c_1\lambda_{\min}}\right)}} \right)  = c_2\left( \log{K} + \log{\log{\left(\frac{K}{c_1}   \cdot \sqrt{\frac{\pi_{max}}{\pi_{min}}} \right)}} \right)
 \numberthis
 \end{align*}

This contributes in the first lower bound ($n_1$) of $N$ in Theorem \ref{thm:bound}.

\section{Tail Inequality}

\begin{lemma}
\label{lemma:tailineq}
If we draw $N$ i.i.d. documents $x_1,x_2 \dots x_N$ through the generative process in Equation \ref{1}, with the labels as  $y_1,y_2 \dots y_N$, and the vectors  probability mass function of the words $v$  estimated from these $N$ samples are $\hat{p}(v)$ whereas the true p.m.f is $p(v)$ with $v \in \{v_1, v_2 \dots v_D\}$ , then with probability at least $1-\delta$ with $\delta \in (0,1)$, 

\begin{align}
\label{eqn:prob1}
\left|\left|\hat {p}(v) - p(v) \right|\right|  & \le \frac{2}{\tilde{d}_{1s}\sqrt{N}}\left( 1 + \sqrt{\frac{\log (1 / \delta)}{2}}\right)\\
\label{eqn:prob2}
 \left|\left|\hat {p}(v,v) - p(v, v) \right|\right|_F  & \le \frac{2}{\tilde{d}_{2s}\sqrt{N}}\left( 1 + \sqrt{\frac{\log (1 / \delta)}{2}}\right)\\
\label{eqn:prob3}
\left|\left|\hat {p}(v, v, v) - p(v, v, v) \right|\right|_F & \le \frac{2}{\tilde{d}_{3s}\sqrt{N}}\left( 1 + \sqrt{\frac{\log (1 / \delta)}{2}}\right)
\end{align}

where, $\tilde{d}_{1s} = \frac{1}{N}\sum_{i =1}^N nnz(x_i)$, $\tilde{d}_{2s} = \frac{1}{N}\sum_{i=1}^N nnz(x_i)^2$, $\tilde{d}_{3s} = \frac{1}{N}\sum_{i=1}^N nnz(x_i)^3$, and $nnz(x_i)$ is the sum of all entries in the row $x_i$ of the data $X$ as described in section \ref{sec:implementation}.
\end{lemma}

\begin{proof}
The generative process in Equation \ref{1}  results in binary sample $x_{1:N}$, with $ |x| = n_d$, where $x$ is the sample corresponding to the document $d$, and $n_d$ is total number of unique words in that document.  From here, we can show that $||x|| \le |x| = n_d$, since $x$ has only positive entries. Since the count of unique words in a document is always bounded, the samples have bounded norm.

Without loss of generality, if we assume $||x|| \le 1$ $ \forall x \in X$, then from Lemma 7 of supplementary material of \cite{slda}, with probability at least $1-\delta$ with $\delta \in (0,1)$,

\begin{align}
\label{eqn:tail1}
\left| \left| \hat{\mathbb{E}}[x] - \mathbb{E}[x]  \right| \right| & \le \frac{2}{\sqrt{N}}\left( 1 + \sqrt{\frac{\log (1 / \delta)}{2}}\right)  \\
\label{eqn:tail2}
\left| \left| \hat{\mathbb{E}}[x \otimes x] - \mathbb{E}[x \otimes x ]  \right| \right|_F  &\le \frac{2}{\sqrt{N}}\left( 1 + \sqrt{\frac{\log (1 / \delta)}{2}}\right)  \\
\label{eqn:tail3}
 \left| \left| \hat{\mathbb{E}}[x \otimes x \otimes x] - \mathbb{E}[x \otimes x \otimes x]  \right| \right|_F  &\le
 \frac{2}{\sqrt{N}}  \left( 1 + \sqrt{\frac{\log (1 / \delta)}{2}}\right) 
\numberthis
\end{align}

where $\mathbb{E}$ stands for expectation from the population, and $\mathbb{\hat{E}}$ stands for the expectation estimated from the $N$ samples, i.e.,
\begin{align*}
\hat{\mathbb{E}}[x] &=\frac{1}{N} \sum_{i=1}^N x_i = \frac{1}{N} X^\top \mathbf{1} \\ 
\hat{\mathbb{E}}[x \otimes x] &=\frac{1}{N} \sum_{i=1}^N x_i \otimes x_i = \frac{1}{N} X^\top X\\
\hat{\mathbb{E}}[x \otimes x \otimes x] &=\frac{1}{N} \sum_{i=1}^N x_i \otimes x_i \otimes x_i = \frac{1}{N} X \otimes X \otimes X
\end{align*}

Now, since each of our samples $x_{1:N}$ contains count data, probability of the words can be estimated from the training data as $\hat{p}(v) = \frac{\mathbb{\hat{E}}[x]}{\sum_v \mathbb{\hat{E}}[x]} $, where $ \sum_v \mathbb{\hat{E}}[x]$ is the sum of $\mathbb{\hat{E}}[x]$ across the dimensions, i.e., all the words. Also, it can be shown that $\sum_v \mathbb{\hat{E}}[x] = \frac{1}{N}\sum_{i=1}^N nnz(x_i)= \tilde{d}_{1s} $. Therefore.  

\begin{equation} \hat{p}(v) = \frac{\mathbb{\hat{E}}[x]}{ \tilde{d}_{1s}}  \end{equation}

Please note that $\sum_v \mathbb{E}[x]  \approx \sum_v \mathbb{\hat{E}}[x] = \tilde{d}_{1s}$, and therefore, 
\begin{equation} \hat{p}(v) - p(v) = \frac{1}{ \tilde{d}_{1s}} ( \mathbb{\hat{E}}[x] - \mathbb{E}[x]), \end{equation}
and using this in Equation \ref{eqn:tail1}, we get the first inequality of the Lemma (Equation \ref{eqn:prob1}).

Since  $\tilde{d}_{2s} =\sum_v \sum_v \mathbb{\hat{E}}[x \otimes x] $ and $ \tilde{d}_{3s} =\sum_v \sum_v \sum_v \mathbb{\hat{E}} [x \otimes x \otimes x] $, the pairwise and triple-wise probability matrices can be estimated as,
 
 \begin{align*}
 \hat{p}(v,v) &= \frac{\mathbb{\hat{E}}[x\otimes x]}{\sum_v \sum_v \mathbb{\hat{E}}[x \otimes x]} =  \frac{\mathbb{\hat{E}}[x\otimes x]}{ \tilde{d}_{2s}}\\
  \hat{p}(v,v,v) &= \frac{\mathbb{\hat{E}}[x \otimes x\otimes x]}{\sum_v \sum_v \sum_v \mathbb{\hat{E}}[x \otimes x \otimes x]} =  \frac{\mathbb{\hat{E}}[x\otimes x \otimes x]}{ \tilde{d}_{3s}}
  \end{align*}
 
Since  $\sum_v \sum_v \mathbb{E}[x \otimes x]  \approx \sum_v \sum_v \hat{\mathbb{E}}[x \otimes x]  = \tilde{d}_{2s}$, and $\sum_v \sum_v \sum_v \mathbb{E}[x \otimes x \otimes x] \approx \sum_v \sum_v \sum_v \hat{\mathbb{E}}[x \otimes x \otimes x]  =  \tilde{d}_{3s}$, we can establish the following equations,

\begin{align*}
\hat{p}(v,v) - p(v,v)  &= \frac{1}{  \tilde{d}_{2s}}\left( \mathbb{\hat{E}}[x\otimes x] -\mathbb{E}[x\otimes x] \right) \\
\hat{p}(v,v,v) - p(v,v,v)  &= \frac{1}{  \tilde{d}_{3s}}\left( \mathbb{\hat{E}}[x\otimes x \otimes x] -\mathbb{E}[x\otimes x \otimes x] \right)
\end{align*}

Substituting these equations in Equation \ref{eqn:tail2} and \ref{eqn:tail3}, we complete the proof.
\end{proof}Also, if $y_i$ represents the label vector associated with $i$th document, whereas $x_i$ represent the word vector,

\begin{align*}
& \hat{\mathbb{E}}[y\otimes x \otimes x] - \mathbb{E}[y\otimes x \otimes x]  \\
& = \frac{1}{N}\sum_{i=1}^N y_i \otimes x_i \otimes x_i - \mathbb{E}[y\otimes x \otimes x] \\
& =  \frac{1}{N}\sum_{i=1}^N y_i \otimes x_i \otimes x_i -  \frac{1}{N} \sum_{i=1}^N y_i \otimes \mathbb{E}[x \otimes x] +  \frac{1}{N}\sum_{i=1}^N y_i \otimes \mathbb{E}[x \otimes x]  -  \mathbb{E}[y\otimes x \otimes x] 
\numberthis
\end{align*}

Therefore,
\begin{align*}
\label{eqn:M2Lineq}
& \left|\left|  \hat{\mathbb{E}}[y\otimes x \otimes x] - \mathbb{E}[y\otimes x \otimes x]  \right| \right|_F  \le \left|\left|\hat{\mathbb{E}}[y]  \right|\right| \left| \left| \hat{\mathbb{E}}[x \otimes x] - \mathbb{E}[x \otimes x] \right| \right|_F + \left| \left| \mathbb{E}[x \otimes x]\right| \right|_F \left| \left| \hat{\mathbb{E}}[y]-\mathbb{E}[y] \right| \right|
\numberthis
\end{align*}

Also, we can safely assume that the number of unique labels assigned to each text in the corpus is bounded. Then similar to Lemma \ref{lemma:tailineq}, without loss of generality, we can assume that $||y|| \le 1$. Then, from Equation \ref{eqn:tail1},
\begin{align*}
P \left[ \left| \left| \hat{\mathbb{E}}[y] - \mathbb{E}[y]  \right| \right|_F \le \frac{2}{\sqrt{N}}\left( 1 + \sqrt{\frac{\log (1 / \delta)}{2}}\right) \right] \ge 1- \delta
\numberthis
\end{align*}

From the assumption $||x|| \le 1$ in Lemma \ref{lemma:tailineq}, \\
\noindent  $||x \otimes x||_F \le 1, \forall x \in X$, and therefore, $|| \mathbb{E} [x \otimes x] ||_F\le 1$. Therefore,

\begin{align*}
P \left[ \left| \left| \hat{\mathbb{E}}[y] - \mathbb{E}[y]  \right| \right|   \Big|\Big| \mathbb{E} [x \otimes x] \Big|\Big|_F   \le \frac{2}{\sqrt{N}}\left( 1 + \sqrt{\frac{\log (1 / \delta)}{2}}\right) \right] \ge 1- \delta
\numberthis
\end{align*}

Also, from the assumption $||y|| \le 1$, $||\mathbb{E}[y] || \le 1$. Then, from Equation \ref{eqn:tail2},  

\begin{align*}
P \left[ \left| \left| \hat{\mathbb{E}}[x \otimes x] - \mathbb{E}[x \otimes x]  \right| \right|_F   \Big|\Big| \mathbb{E} [y] \Big|\Big|  \le  \frac{2}{\sqrt{N}}\left( 1 + \sqrt{\frac{\log (1 / \delta)}{2}}\right) \right] \ge 1- \delta
\numberthis
\end{align*}

Reversing the probability bounds,

\begin{align*}
\label{eqn:ineq1}
P \left[ \left| \left| \hat{\mathbb{E}}[y] - \mathbb{E}[y]  \right| \right|   \Big|\Big| \mathbb{E} [x \otimes x] \Big|\Big|_F  \ge \frac{2}{\sqrt{N}}\left( 1 + \sqrt{\frac{\log (1 / \delta)}{2}}\right) \right] \le  \delta
\numberthis
\end{align*}

\noindent and,

\begin{align*}
\label{eqn:ineq2}
P \left[ \left| \left| \hat{\mathbb{E}}[x \otimes x] - \mathbb{E}[x \otimes x]  \right| \right|_F   \Big|\Big| \mathbb{E} [y] \Big|\Big|  \ge \frac{2}{\sqrt{N}}\left( 1 + \sqrt{\frac{\log (1 / \delta)}{2}}\right) \right] \le \delta
\numberthis
\end{align*}

For two events $\mathcal{E}_1$ and $\mathcal{E}_2$, \\
\begin{align*}
P(\mathcal{E}_1 \cap \mathcal{E}_2) & = P(\mathcal{E}_1) + P(\mathcal{E}_1) - P(\mathcal{E}_1\cup \mathcal{E}_2) \\
& \le   P(\mathcal{E}_1) + P(\mathcal{E}_1)
\end{align*}

Therefore, from the above two inequalities \ref{eqn:ineq1} and \ref{eqn:ineq2}, 

\begin{align*}
& P \Big[ \left|\left|\hat{\mathbb{E}}[y]  \right|\right| \left| \left| \hat{\mathbb{E}}[x \otimes x] - \mathbb{E}[x \otimes x] \right| \right|_F + \left| \left| \mathbb{E}[x \otimes x]\right| \right|_F \left| \left| \hat{\mathbb{E}}[y]-\mathbb{E}[y] \right| \right|  \ge \frac{4}{\sqrt{N}}\left( 1 + \sqrt{\frac{\log (1 / \delta)}{2}}\right) \Big]  \le 2\delta
\end{align*}

\noindent Or,

\begin{align*}
& P \Big[ \left|\left|\hat{\mathbb{E}}[y]  \right|\right| \left| \left| \hat{\mathbb{E}}[x \otimes x] - \mathbb{E}[x \otimes x] \right| \right|_F  + \left| \left| \mathbb{E}[x \otimes x]\right| \right|_F \left| \left| \hat{\mathbb{E}}[y]-\mathbb{E}[y] \right| \right|  \le \frac{4}{\sqrt{N}}\left( 1 + \sqrt{\frac{\log (1 / \delta)}{2}}\right) \Big]  \ge 1 - 2\delta
\end{align*}

Using Equation \ref{eqn:M2Lineq}, and replacing $\delta$ by $\delta/2$,

\begin{align*}
&P\Big[ \left|\left|  \hat{\mathbb{E}}[y\otimes x \otimes x] - \mathbb{E}[y\otimes x \otimes x]  \right| \right|_F 
 \le \frac{4}{\sqrt{N}}\left( 1 + \sqrt{\frac{\log (2 / \delta)}{2}}\right)\Big] \ge 1-\delta
\end{align*}

Let us note that,
\begin{align*}
\hat{p}(l,v,v) = \frac{\mathbb{\hat{E}}[y \otimes x\otimes x]}{\sum_l \sum_v \sum_v \mathbb{\hat{E}}[y \otimes x \otimes x ]} =  \frac{\mathbb{\hat{E}}[y \otimes x \otimes x]}{ \tilde{d}_{ls}}
\numberthis
\end{align*}

Also, since $\sum_l \sum_v \sum_v \mathbb{E}[y \otimes x \otimes x]  \approx \sum_l \sum_v \sum_v \hat{\mathbb{E}}[y \otimes x \otimes x]  = \tilde{d}_{ls}$, where $\tilde{d}_{ls} = \frac{1}{N}\sum_{i=1}^N nnz(y_i)nnz(x_i)^2$, and $nnz(y_i)$ is the number of labels associated with the $i$th document, in a similar way to the proof of Lemma \ref{lemma:tailineq}, we can prove that with probability at least $1-\delta$,

\begin{align*}
\left|\left|\hat {p}(l, v, v) - p(l, v, v) \right|\right|_F \le \frac{4}{\tilde{d}_{ls}\sqrt{N}}\left( 1 + \sqrt{\frac{\log (2 / \delta)}{2}}\right)
\numberthis
\end{align*}

Since by definition $M_2=p(v,v)$, $M_3=p(v,v,v)$, $M_{2L}=p(l,v,v)$, $\varepsilon_{M_2}=||M_2 -\hat{M}_2||_2$, $\varepsilon_{M_3}=||M_3 -\hat{M}_3||_2$ and $\varepsilon_{M_{2L}}=||M_{2L} -\hat{M}_{2L}||_2$, assigning $\varepsilon_1 = \left( 1 + \sqrt{\frac{\log (1 / \delta)}{2}}\right)$, and  $\varepsilon_2 = \left( 1 + \sqrt{\frac{\log (2 / \delta)}{2}}\right)$, the following inequalities hold with probability at least $1-\delta$,
 
\begin{align*}
&\varepsilon_{M_2}\le ||p(v,v)-\hat{p}(v,v)||_F \le \frac{2\varepsilon_1}{\tilde{d}_{2s}\sqrt{N}}\\
&\varepsilon_{M_3} \le ||p(v,v,v)-\hat{p}(v,v,v)||_F  \le \frac{2\varepsilon_1}{\tilde{d}_{3s}\sqrt{N}}\\ 
&\varepsilon_{M_{2L}} \le ||p(l,v,v)-\hat{p}(l,v,v)||_F  \le \frac{4\varepsilon_2}{\tilde{d}_{ls}\sqrt{N}}
\end{align*}

\noindent since operator norm is smaller than Frobenius norm.
\newline

Also, to satisfy $ \varepsilon_{M_2} \le \sigma_K(M_2)/2$, we need,
\begin{equation}
N \ge \Omega\left( \left(  \frac{1}{\tilde{d}_{2s}\sigma_K(M_2)}\left( 1 + \sqrt{\frac{\log (1 / \delta)}{2}}\right)  \right)^2 \right)
\end{equation}

Or, $N \ge \Omega\left( \left(  \frac{\varepsilon_1}{\tilde{d}_{2s}\sigma_K(M_2)} \right)^2 \right)$. This contributes in the second lower bound ($n_2$) of $N$ in Theorem \ref{thm:bound}.

Also, from Equation \ref{eqn:etw}, 
\begin{equation}
\varepsilon_{tw} \le \left( \frac{10}{\tilde{d}_{2s}\sigma_K(M_2)^{5/2}}  + \frac{2\sqrt{2}}{\tilde{d}_{3s}\sigma_K(M_2)^{3/2}} \right) \frac{2\varepsilon_1}{\sqrt{N}} 
\end{equation}

From Lemma \ref{lemma:rpm}, $\epsilon \le c_1\cdot (\lambda_{\min}/K)$, and we can assign $\epsilon$ as the upper bound of $\varepsilon_{tw}$. To satisfy this, we need, 
\begin{align*}
& \left( \frac{10}{\tilde{d}_{2s}\sigma_K(M_2)^{5/2}}  + \frac{2\sqrt{2}}{\tilde{d}_{3s}\sigma_K(M_2)^{3/2}} \right) \frac{2\varepsilon_1}{\sqrt{N}} \le c_1\frac{\lambda_{\min}}{K} 
\text{, or,} \\
& \left( \frac{10}{\tilde{d}_{2s}\sigma_K(M_2)^{5/2}}  + \frac{2\sqrt{2}}{\tilde{d}_{3s}\sigma_K(M_2)^{3/2}} \right) \frac{2\varepsilon_1}{\sqrt{N}} \le c_1\frac{1}{K\sqrt{\pi_{\max}}}
\end{align*}

Since $\pi_{\max} \le 1$, we need

\begin{align*}
N \ge \Omega\left( K^2 \left( \frac{10}{\tilde{d}_{2s}\sigma_K(M_2)^{5/2}}  + \frac{2\sqrt{2}}{\tilde{d}_{3s}\sigma_K(M_2)^{3/2}} \right)^2 \varepsilon_1^2 \right)
\end{align*}

This contributes to $n_3$ in Theorem \ref{thm:bound}.

\section{Completing the Proof}

Here, we will derive the final bounds for the reconstruction error for the parameters. Since $\mu_k =  W^\dagger u_k $ (Algorithm \ref{alg:mom}), with probability at least $1 - \delta$,
\begin{align*}
&||\mu_k-\hat{\mu}_k||\\
& = ||W^\dagger u_k - \hat{W}^\dagger \hat{u}_k|| \\
&= ||W^\dagger u_k  -W^\dagger \hat{u}_k +W^\dagger \hat{u}_k-\hat{W}^\dagger \hat{u}_k|| \\
& \le ||W^\dagger||_2||u_k - \hat{u}_k|| + ||W^\dagger - \hat{W}^\dagger||_2|| \hat{u}_k|| \\
& \le ||W^\dagger||_2 \frac{8\epsilon}{\lambda_k} + \varepsilon_{W^\dagger}    \\
& \le  8\sqrt{\sigma_1(M_2)}\epsilon + \frac{2 \sqrt{\sigma_1(M_2)} }{\sigma_K \left( M_2 \right)} \varepsilon_{M_2} \\
\numberthis
\end{align*}

Since $\frac{1}{\lambda_k} = \sqrt{\pi_k} \le 1$.  Assigning $\epsilon$ as the upper bound of $\varepsilon_{tw}$ in equation \ref{eqn:etw}, with probability at least $1 - \delta$,
\begin{align*}
&||\mu_k-\hat{\mu}_k|| \\
& \le  8\sqrt{\sigma_1(M_2)} \left( \frac{10}{\tilde{d}_{2s}\sigma_K(M_2)^{5/2}}  + \frac{2\sqrt{2}}{\tilde{d}_{3s}\sigma_K(M_2)^{3/2}} \right)    +  \frac{2\varepsilon_1}{\sqrt{N}}  +   \frac{2 \sqrt{\sigma_1(M_2)} }{\sigma_K \left( M_2 \right)} \frac{2\varepsilon_1}{\tilde{d}_{2s}\sqrt{N}}\\
& \le   \left( \frac{160\sqrt{\sigma_1(M_2)}}{\tilde{d}_{2s}\sigma_K(M_2)^{5/2}}  + \frac{32\sqrt{2\sigma_1(M_2)}}{\tilde{d}_{3s}\sigma_K(M_2)^{3/2}} + \frac{4 \sqrt{\sigma_1(M_2)} }{\tilde{d}_{2s}\sigma_K \left( M_2 \right)} \right) \frac{\varepsilon_1}{\sqrt{N}} \\
\numberthis
\end{align*}

Similarly, since $\pi_k \le 1$, with probability at least $1 - \delta$, 
\begin{align*}
|\pi_k-\hat{\pi}_k| &=\left| \frac{1}{\lambda_k^2}- \frac{1}{\hat{\lambda}_k^2} \right| = \left| \frac{(\lambda_k+\hat{\lambda}_k)(\lambda_k-\hat{\lambda}_k)}{\lambda_k^2 \hat{\lambda}_k^2} \right| \\
&= \left| \sqrt{\pi_k \hat{\pi}_k} \left( \sqrt{\pi_k} + \sqrt{\hat{\pi}_k}  \right) (\lambda_k- \hat{\lambda}_k) \right| \\
& \le 2 | \lambda_k- \hat{\lambda}_k | \le 10\epsilon
\end{align*}

since $ | \lambda_k- \hat{\lambda}_k | \le 5\epsilon$ from Lemma \ref{lemma:rpm}. Therefore, with probability at least $1 - \delta$, we get
\begin{align*}
|\pi_k-\hat{\pi}_k| \le \left( \frac{200}{\sigma_K(M_2)^{5/2}}+ \frac{40\sqrt{2}}{\sigma_K(M_2)^{3/2}} \right)\frac{\varepsilon_1}{\tilde{d}_{3s}\sqrt{N}} 
\end{align*}

where $\varepsilon_1 = \left( 1 + \sqrt{\frac{\log (1 / \delta)}{2}}\right)$.

Also, since $\gamma_k = u_k^\top M_{2L}(W,W) u_k$, with probability at least $1 - \delta$,

\begin{align*}
&||\gamma_k-\hat{\gamma}_k|| \\
& = \left| \left| u_k^\top M_{2L}(W,W) u_k - \hat{u}_k^\top \hat{M}_{2L}(\hat{W},\hat{W}) \hat{u}_k \right| \right| \\
& \le \left| \left| u_k^\top M_{2L}(W,W) u_k - \hat{u}_k^\top M_{2L}(W,W) \hat{u}_k \right| \right| + \left| \left| \hat{u}_k^\top M_{2L}(W,W) \hat{u}_k - \hat{u}_k^\top \hat{M}_{2L}(\hat{W},\hat{W}) \hat{u}_k \right| \right| \\ 
&  \le \left| \left| u_k^\top M_{2L}(W,W) u_k - \hat{u}_k^\top M_{2L}(W,W) u_k+ \hat{u}_k^\top M_{2L}(W,W) u_k - \hat{u}_k^\top M_{2L}(W,W) \hat{u}_k \right| \right|  \\ 
& \phantom{u_k^\top M_{2L}(W,W) u_k - \hat{u}_k^\top M_{2L}(W,W) u_k 111111111111} + \left| \left| \hat{u}_k^\top M_{2L}(W,W) \hat{u}_k - \hat{u}_k^\top \hat{M}_{2L}(\hat{W},\hat{W}) \hat{u}_k \right| \right| \\ 
& \le ||u_k - \hat{u}_k|| ||u_k|| \left|\left| M_{2L}(W,W) \right|\right|_2  + ||u_k - \hat{u}_k|| ||\hat{u}_k|| \left|\left| M_{2L}(W,W) \right|\right|_2\\
& \phantom{u_k^\top M_{2L}(W,W) u_k - \hat{u}_k^\top M_{2L}(W,W) u_k 11111111 11111} + ||\hat{u}_k||^2 \left| \left|  M_{2L}(W,W) -  \hat{M}_{2L}(\hat{W},\hat{W}) \right| \right|_2 \\
& \le 2||u_k - \hat{u}_k|| ||W||^2 \left|\left| M_{2L}\right|\right|_2   +  \left| \left|  M_{2L}(W,W) -  \hat{M}_{2L}(\hat{W},\hat{W}) \right| \right|_2  \\
\end{align*}

Now,
\begin{align*}
& \left| \left|  M_{2L}(W,W) -  \hat{M}_{2L}(\hat{W},\hat{W}) \right| \right|_2 \\
&= \left| \left|  M_{2L}(W,W) - M_{2L}(\hat{W},\hat{W}) + M_{2L}(\hat{W},\hat{W})-  \hat{M}_{2L}(\hat{W},\hat{W}) \right| \right|_2 \\
&\le  \left| \left|  M_{2L}(W,W) - M_{2L}(\hat{W},\hat{W}) \right|\right|_2 +  \left| \left|   M_{2L}(\hat{W},\hat{W})-  \hat{M}_{2L}(\hat{W},\hat{W})  \right|\right|_2 \\
&\le  \left| \left|  M_{2L}(W,W) -  M_{2L}(W,\hat{W})+ M_{2L}(W,\hat{W}) + M_{2L}(\hat{W},\hat{W}) \right|\right|_2 + ||\hat{W}||^2\varepsilon_{M_{2L}}\\
&\le ||M_{2L}||_2 \left( ||W||_2 + ||\hat{W}||_2 \right)\varepsilon_W + ||\hat{W}||^2\varepsilon_{M_{2L}}
\end{align*}

The individual elements of $M_{2L}$ are probabilities whose sum is 1. Therefore, $ ||M_{2L}||_F \le 1$, and, from Lemma \ref{lemma:tnorm}, $||M_{2L}||_2 \le ||M_{2L}||_F \le 1$. Therefore,

\begin{align*}
&||\gamma_k-\hat{\gamma}_k|| \\
& \le  2||u_k - \hat{u}_k|| ||W||^2 + \left( ||W||_2 + ||\hat{W}||_2 \right)\varepsilon_W + ||\hat{W}||^2\varepsilon_{M_{2L}} \\
& \le  16\frac{\epsilon}{\lambda_k} ||W||^2 + \left( ||W||_2 + ||\hat{W}||_2 \right)\varepsilon_W + ||\hat{W}||^2\varepsilon_{M_{2L}}
\end{align*}

Also, $1/\lambda_k =\sqrt{\pi_k} \le 1$.

Assigning $\epsilon$ as the upper bound of $\varepsilon_{tw}$ in equation \ref{eqn:etw}, with probability at least $1 - \delta$,
 
\begin{align*}
&||\gamma_k-\hat{\gamma}_k|| \\
& \le  16\varepsilon_{tw} ||W||^2 + \left( ||W||_2 + ||\hat{W}||_2 \right)\varepsilon_W + ||\hat{W}||^2\varepsilon_{M_{2L}}\\
& \le  16\frac{\varepsilon_{tw}}{ \sigma_K(M_2)} + \frac{\left(1+\sqrt{2}\right)}{\sqrt{\sigma_K(M_2)}}  \frac{2}{\sigma_K(M_2)^{3/2}} \varepsilon_{M2} + \frac{2}{\sigma_K(M_2)}\varepsilon_{M_{2L}}\\
& \le  16\left( \frac{10}{\tilde{d}_{2s}\sigma_K(M_2)^{7/2}}  + \frac{2\sqrt{2}}{\tilde{d}_{3s}\sigma_K(M_2)^{5/2}} \right) \frac{2\varepsilon_1}{\sqrt{N}} + \frac{2\left(1+\sqrt{2}\right)}{\sigma_K(M_2)^2}\frac{2\varepsilon_1}{\tilde{d}_{2s}\sqrt{N}}  + \frac{2}{\sigma_K(M_2)}\frac{4\varepsilon_2}{\tilde{d}_{ls}\sqrt{N}} \\
 & =  \left( \frac{160}{\tilde{d}_{2s}\sigma_K(M_2)^{7/2}}  + \frac{32\sqrt{2}}{\tilde{d}_{3s}\sigma_K(M_2)^{5/2}} +  \frac{ 2+2\sqrt{2}}{\tilde{d}_{2s}\sigma_K(M_2)^2} \right) \frac{2\varepsilon_1}{\sqrt{N}}  + \frac{8\varepsilon_2}{\tilde{d}_{ls}\sigma_K(M_2)\sqrt{N}}
\end{align*}
where $\varepsilon_1 = \left( 1 + \sqrt{\frac{\log (1 / \delta)}{2}}\right)$ and $\varepsilon_2 = \left( 1 + \sqrt{\frac{\log (2 / \delta)}{2}}\right)$.
This completes the proof of Theorem \ref{thm:bound}.

\end{document}